\documentclass{comnet}

\usepackage[english]{babel}
\usepackage[utf8]{inputenc}

\usepackage{amsmath,amsfonts,bm}

\def\eqref#1{equation~\ref{#1}}

\def\1{\bm{1}}

\newcommand{\test}{\mathcal{D_{\mathrm{test}}}}

\def\vc{{\textbf{c}}}

\def\vw{{\textbf{w}}}

\def\mA{{\textbf{A}}}

\def\mC{{\textbf{C}}}
\def\mD{{\textbf{D}}}
\def\mE{{\textbf{E}}}
\def\mF{{\textbf{F}}}
\def\mG{{\textbf{G}}}
\def\mH{{\textbf{H}}}
\def\mI{{\textbf{I}}}

\def\mP{{\textbf{P}}}

\def\mW{{\textbf{W}}}

\DeclareMathAlphabet{\mathsfit}{\encodingdefault}{\sfdefault}{m}{sl}
\SetMathAlphabet{\mathsfit}{bold}{\encodingdefault}{\sfdefault}{bx}{n}

\def\gF{{\mathcal{F}}}
\def\gG{{\mathcal{G}}}

\def\gS{{\mathcal{S}}}
\def\gT{{\mathcal{T}}}

\def\sD{{\mathbb{D}}}
\def\sF{{\mathbb{F}}}
\def\sE{{\mathbb{E}}}

\def\sR{{\mathbb{R}}}

\def\sV{{\mathbb{V}}}

\usepackage[caption=false]{subfig}
\usepackage{color}
\usepackage{xcolor}
\usepackage{amssymb}
\usepackage{amsmath}
\usepackage{amsthm}
\usepackage{mathtools}
\usepackage{thmtools}
\usepackage{thm-restate}
\usepackage[linesnumbered,lined,boxed,algo2e]{algorithm2e}
\usepackage{pgfplots}
\usepackage{tikz}

\usetikzlibrary{arrows,petri,topaths,backgrounds,decorations,patterns,positioning,pgfplots.groupplots}
\usetikzlibrary{calc,patterns,positioning,decorations.pathmorphing}

\usepackage{tkz-berge}
\usepackage{pifont}

\SetAlCapSkip{1em}

\definecolor{cycle2}{RGB}{106, 191, 0}
\definecolor{cycle3}{RGB}{191, 0, 0}
\pgfplotsset{compat=1.17}

\newcommand{\cmark}{\textcolor{cycle2}{\ding{52}}}
\newcommand{\xmark}{\textcolor{cycle3}{\ding{56}}}

\newcommand{\specialcell}[2][c]{  \begin{tabular}[#1]{@{}c@{}}#2\end{tabular}}

\newcommand{\keypoint}[1]{\noindent\textbf{#1}\ }

\newcommand{\usecond}[1]{\text{#1}}

\begin{document}

\title{Multi-Scale Attributed Node Embedding}

\shorttitle{Multi-Scale Attributed Node Embedding} 
\shortauthorlist{Rozemberczki et al.} 

\author{
\name{Benedek Rozemberczki$^*$}
\address{School of Informatics, The University of Edinburgh, 10 Crichton St, Newington, Edinburgh EH8 9AB, United Kingdom\email{$^*$benedek.rozemberczki@ed.ac.uk}}
\name{Carl Allen}
\address{School of Informatics, The University of Edinburgh, 10 Crichton St, Newington, Edinburgh EH8 9AB, United Kingdom}
\and
\name{Rik Sarkar}
\address{School of Informatics, The University of Edinburgh, 10 Crichton St, Newington, Edinburgh EH8 9AB, United Kingdom}}

\maketitle

\begin{abstract}
{We present network embedding algorithms that capture information about a node from the local distribution over node attributes around it, as observed over random walks following an approach similar to Skip-gram. Observations from neighborhoods of different sizes are either pooled (AE) or encoded distinctly in a multi-scale approach (MUSAE). Capturing attribute-neighborhood relationships over multiple scales is useful for a range of applications, including latent feature identification across disconnected networks with similar features. We prove theoretically that matrices of node-feature pointwise mutual information are implicitly factorized by the embeddings. Experiments show that our algorithms are computationally efficient and outperform comparable models on social networks and web graphs.}
{node embedding, node classification, attributed network, dimensionality reduction.}

\end{abstract}

\section{Introduction}

Node embedding is a fundamental technique in network analysis that serves as a precursor to numerous downstream machine learning and optimisation tasks, e.g. community detection, network visualization and link prediction~\cite{deepwalk,node2vec,tang2015line}. Several recent network embedding methods, such as \textit{Deepwalk}~\cite{deepwalk}, and \textit{Walklets}~\cite{perozzidontwalk}, achieve impressive performance by learning the network structure following an approach similar to {\em Word2Vec Skip-gram}~\cite{mikolov_2}, originally designed for word embedding. In these works, sequences of neighboring nodes are generated from random walks over a network, and representations are distilled from node-node proximity statistics that capture local neighbourhood information.

In real-world networks, nodes often have attributes (or features), e.g. in a social network nodes may represent people and node attributes may capture a person's interests, preferences or habits.
Attributes of a node and those of its local neighbourhood may contain information useful in downstream tasks. Such neighborhoods can be considered at different path lengths, or \textit{scales}, e.g. in a social network, near neighbors may be friends, whereas nodes separated by greater scales may have looser friends-of-friends associations.  Attributes of neighbors at different scales can be considered separately (\textit{multi-scale}) or \textit{pooled} in some way, e.g. weighted average. Node attributes can identify different network structure, e.g. nodes with similar attributes are often more likely to be connected (known as \textit{homophily}) such that patterns of similar node attributes may identify a community. Alternatively, similar attribute distributions over node neighbourhoods may identify similar network \textit{roles} even at distant locations in a network (Figure~\ref{fig:scaleness}), or in different networks. 

\textit{Attributed} network embedding methods~\cite{yang2015network,huang2017accelerated,liao2018attributed,feather} leverage attribute information to supplement local network structure, benefiting many applications, e.g. recommender systems, node classification and link prediction~\cite{yang2018binarized,yang2018enhanced,zhang2018sine}. Methods that consider a node's attributes generalize those that do not, for which a node's ``feature'' can be considered a standard basis vector (i.e. the node-feature matrix is the identity matrix).

\vspace{-10mm}
\begin{figure}[h!]
\hspace{-10mm}
\subfloat[Attributed example graph.\label{fig:scaleness}]{
\centering
\tikzset{VertexStyle/.style = {draw, very thick, circle}}
\begin{tikzpicture}[scale=0.63,transform shape]
\node[] at (-1,4.6) {{\large$\left \{ \clubsuit; \heartsuit; \diamondsuit  \right \}$}};
\node[] at (-3,-0.7) {{\large$\left \{ \heartsuit \right \}$}};
\node[] at (-1,-0.7) {{\large$\left \{ \clubsuit;\diamondsuit \right \}$}};
\node[] at (-3.6,2.6) {{\large$\left \{ \heartsuit;\spadesuit \right \}$}};

\node[] at (2,4.6) {{\large$\left \{ \clubsuit; \heartsuit; \diamondsuit; \spadesuit  \right \}$}};
\node[] at (1.6,2.6) {{\large$\left \{\diamondsuit; \spadesuit  \right \}$}};
\node[] at (4.8,2.0) {{\large$\left \{ \spadesuit  \right \}$}};
\node[] at (2,-0.7) {{\large$\left \{ \clubsuit \right \}$}};
\node[] at (6,4.6) {{\large$\left \{ \heartsuit;\spadesuit \right \}$}};
\node[] at (6,-0.7) {{\large$\left \{  \clubsuit;\diamondsuit  \right \}$}};
\node[] at (8,2.6) {{\large$\left \{ \heartsuit \right \}$}};
  \Vertex[L=$\textbf{A}$,x=-1,y=0]{1}
  \Vertex[L=$\textbf{B}$,x=1,y=2]{2}
  \Vertex[L=$\textbf{C}$,x=-1,y=4]{3}
  \Vertex[L=$\textbf{D}$,x=-3,y=0]{4}
  \Vertex[L=$\textbf{E}$,x=-3,y=2]{5}
  \Vertex[L=$\textbf{F}$,x=4,y=2]{6}
  \Vertex[L=$\textbf{G}$,x=8,y=2]{9} 
  \Vertex[L=$\textbf{H}$,x=6,y=0]{8}
    \Vertex[L=$\textbf{I\,}$,x=6,y=4]{10}
    \Vertex[L=$\textbf{J}$,x=2,y=0]{7} 
    \Vertex[L=$\textbf{K}$,x=2,y=4]{11}       
  \tikzstyle{LabelStyle}=[fill=white,sloped]
\tikzset{EdgeStyle/.style = {draw, very thick, circle, fill=yellow!40}}
  \Edge(1)(2)
  \Edge(2)(3)
  \Edge(1)(4)
  \Edge(1)(5)
  \Edge(4)(5)
  \Edge(2)(5)
  \Edge(3)(5) 
  
   \Edge(2)(6)
   \Edge(7)(6)
   \Edge(11)(6)
   \Edge(10)(6)
   \Edge(10)(9)
\Edge(8)(9)
\Edge(6)(8)
\Edge(10)(8)
\end{tikzpicture}}
\subfloat[Densification of the target matrix.\label{fig:powers}]{
  \centering

	\begin{tikzpicture}[scale=0.67,transform shape]
	
	\begin{axis}[
	width=0.65\textwidth,
	height=0.45\textwidth,	
	grid=major,
	grid style={dashed, gray!40},
	scaled ticks=false,
 	inner axis line style={-stealth},
    legend columns=1,	
	ytick={0.0,0.2,0.4,0.6,0.8,1.0},
	xtick={1,2,3,4,5,6,7,8},
	xlabel=$r$,
	ylabel=Ratio of non-zero elements in $\textbf{A}^r \textbf{X}$,
	enlargelimits=0.1,
	legend style={at={(0.95,0.35)},anchor=east}
	]
	
	\addplot[mark=triangle*,opacity=0.8,mark options={black,fill=yellow},mark size=5pt]
	coordinates {
(1,0.017)
(2,0.068)
(3,0.206)
(4,0.469)
(5,0.734)
(6,0.898)
(7,0.969)
(8,0.992)

	};\addlegendentry{Facebook}%
	\addplot[mark=diamond*,opacity=0.8,mark options={black,fill=blue},mark size=5pt]
	coordinates {
(1,0.023)
(2,0.467)
(3,0.894)
(4,0.983)
(5,0.997)
(6,0.999)
(7,1.0)
(8,1.0)

	};\addlegendentry{GitWebML}%
	\addplot[mark=*,opacity=0.8,mark options={black,fill=red},mark size=3pt]
	coordinates {
(1,0.073)
(2,0.479)
(3,0.877)
(4,0.986)
(5,0.998)
(6,1.0)
(7,1.0)
(8,1.0)

	};\addlegendentry{Twitch ES}%
	
	\addplot[mark=square*,opacity=0.8,mark options={black,fill=green},mark size=3pt]
	coordinates {
(1,0.067)
(2,0.449)
(3,0.838)
(4,0.977)
(5,0.998)
(6,1.0)
(7,1.0)
(8,1.0)
		
	};\addlegendentry{Twitch RU}%
	\addplot[mark=triangle*,opacity=0.8,mark options={black,fill=orange},mark size=5pt]
	coordinates {
(1,0.092)
(2,0.483)
(3,0.746)
(4,0.919)
(5,0.973)
(6,0.993)
(7,0.998)
(8,1.0)

	};\addlegendentry{Chameleons}%
	\addplot[mark=diamond*,opacity=0.8,mark options={black,fill=purple},mark size=5pt]
	coordinates {
(1,0.132)
(2,0.517)
(3,0.812)
(4,0.95)
(5,0.99)
(6,0.998)
(7,1.0)
(8,1.0)};\addlegendentry{Squirrels}

\end{axis}
\end{tikzpicture}}

\caption{Phenomena affecting and inspiring the design of the multi-scale attributed network embedding procedure. In \textbf{Figure \ref{fig:scaleness}} attributed nodes D and G have the same feature set and their nearest neighbours also exhibit equivalent sets of features, whereas features at higher order neighbourhoods differ.  A multi-scale attributed node embedding method is able to represent these differences and similarities in the embedding space. 
\textbf{Figure \ref{fig:powers}} shows that as the order of neighbourhoods considered (\textit{r}) increases, the product of the adjacency matrix  power and the feature matrix becomes less sparse. This suggests that an implicit decomposition method would be computationally beneficial for learning an embedding.}

\end{figure}
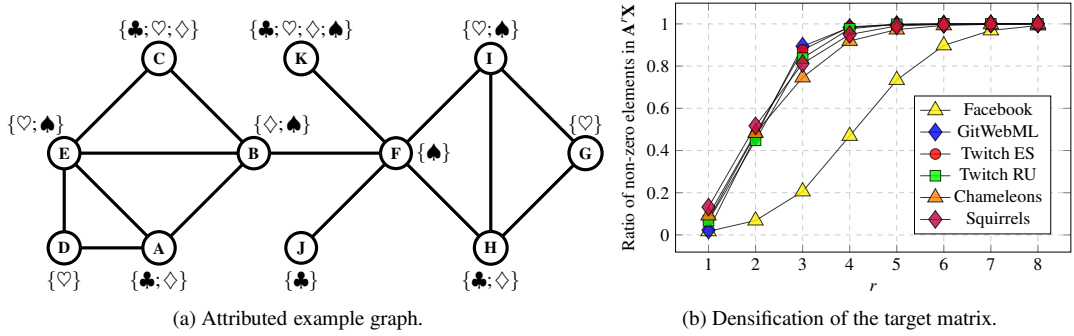

Many embedding methods correspond to matrix factorization, indeed some attributed embedding methods ~\cite{yang2018binarized} explicitly factorize a matrix of node-attribute information. 
Word embeddings learned using Skip-gram are known to implicitly factorize a matrix of \textit{pointwise mutual information} (PMI) of word co-occurrence statistics~\cite{levy2014neural}. Related network embedding methods~\cite{deepwalk,node2vec,tang2015line,qiu2018network} thus also factorize PMI matrices that relate to the probability of encountering other nodes on a random walk~\cite{qiu2018network}. Our key contributions are:

\begin{enumerate}
    \item We introduce Skip-gram style embedding algorithms that consider attribute distributions over local neighborhoods, both pooled (\textit{AE}) and multi-scale (\textit{MUSAE}), and their counterparts with distinct proximal features attributed to nodes (\textit{AE-EGO} and \textit{MUSAE-EGO}).
    
    \item We derive that the PMI matrices factorised by all embeddings in terms of adjacency and node-feature matrix and show that popular network embedding methods \textit{DeepWalk} \cite{deepwalk} and \textit{Walklets} \cite{perozzidontwalk} are special cases of \textit{AE} and \textit{MUSAE}.
	\item We show empirically that on real-world networks our algorithms outperform comparable methods at predicting node attributes,  computationally scalable and enable \textit{transfer learning}. 
    \item We provide reference implementations of \textit{AE} and \textit{MUSAE}, together with the datasets used for evaluation at \textit{https://github.com/benedekrozemberczki/MUSAE}. The proposed embedding procedures are also available in the open source \textit{Karate Club} machine learning library \cite{karateclub}.
\end{enumerate}

\begin{table}[!h]

\setlength{\tabcolsep}{3pt}
\centering
\caption{A summary of existing node embedding techniques (proximity preserving and attributed) with respect to having (\cmark) and missing (\xmark) desired properties. The time and space complexities are reported as a function of vertex and edge counts ($|\sV|$ and $|\sE|$), unique feature count $|\sF|$, average node feature count $m$, and embedding dimensions $d$.}

\label{fig:musae_comparison}
{\footnotesize
\begin{tabular}{l ccccccccccc}
         & \specialcell{\textbf{Generic}\\\textbf{Features}} &  \specialcell{\textbf{Multi}\\\textbf{Scale}} & \textbf{Implicit} & \textbf{Proximal} & \specialcell{\textbf{Higher}\\ \textbf{Order}} & \textbf{Inductive} & \textbf{Non-linear} &\specialcell{\textbf{Space}\\\textbf{Complexity}} & \specialcell{\textbf{Time}\\\textbf{Complexity}} \\ \hline
\textbf{DeepWalk} \cite{deepwalk}&\xmark&\xmark&\cmark&\cmark&\cmark&\xmark&\cmark&$\mathcal{O}(|\sV| \ d)$  &  $\mathcal{O}(|\sV|\ d)  $\\ [0.35em]
\textbf{LINE}$_2$ \cite{tang2015line}&\xmark&\cmark&\cmark&\cmark&\cmark&\xmark&\cmark&$\mathcal{O}(|\sV| \ d)$  &  $\mathcal{O}(|\sV|\ d)  $\\ [0.35em]
\textbf{Node2Vec} \cite{node2vec}&\xmark&\xmark&\cmark&\cmark&\cmark&\xmark&\cmark&$\mathcal{O}(|\sV|^3)$    &  $\mathcal{O}(|\sV|\ d)  $\\ [0.35em]
\textbf{Walklets} \cite{perozzidontwalk}&\xmark&\cmark&\cmark&\cmark&\cmark&\xmark&\cmark&$\mathcal{O}(|\sV| \ d)$  &  $\mathcal{O}(|\sV| \ d)  $\\ [0.35em]
\textbf{NetMF} \cite{qiu2018network}&\xmark&\xmark&\xmark&\cmark&\cmark&\xmark&\xmark&$\mathcal{O}(|\sV|^2\ d)$  &  $\mathcal{O}(|\sV|^3 \ d)  $\\  [0.35em]
\textbf{HOPE} \cite{hope}&\xmark&\xmark&\xmark&\cmark&\cmark&\xmark&\xmark&$\mathcal{O}(|\sV|^2\ d)$  &  $\mathcal{O}(|\sV|^3 \ d)  $\\  [0.35em]
\textbf{GraRep} \cite{cao2015grarep}&\xmark&\cmark&\xmark&\cmark&\cmark&\xmark&\xmark&$\mathcal{O}(|\sV|^2 \ d)$  &  $\mathcal{O}(|\sV|^3 \ d)  $\\ [0.35em]
\hline
\textbf{TADW} \cite{yang2015network}&\cmark&\xmark&\xmark&\xmark&\cmark&\xmark&\xmark&$\mathcal{O}(|\sV| + |\sF|) \ d )$&$\mathcal{O}(|\sV|^2 \ |\sF| \ d )$\\[0.35em]
\textbf{ASNE} \cite{liao2018attributed}&\cmark&\xmark&\xmark&\cmark&\xmark&\xmark&\xmark&$\mathcal{O}(|\sV| + |\sF|) \ d )$&$\mathcal{O}(|\sE| \ m \ d )$\\[0.35em]
\textbf{AANE} \cite{huang2017accelerated}&\cmark&\xmark&\xmark&\cmark&\xmark&\xmark& \cmark&$\mathcal{O}(|\sV|^2 \ m \ d )$&$\mathcal{O}(|\sV|^2 \ m \ d )$\\[0.35em]
\textbf{BANE} \cite{yang2018binarized} &\cmark&\xmark&\xmark&\xmark&\cmark&\xmark&\xmark&$\mathcal{O}(|\sV|^2 \ m\ d )$&$\mathcal{O}(|\sV|^3 \ m\ d )$\\[0.35em]
\textbf{TENE} \cite{yang2018enhanced} &\cmark&\xmark&\xmark&\cmark&\xmark&\xmark&\xmark&$\mathcal{O}(|\sV| + |\sF|) \ d )$&$\mathcal{O}(|\sE| \ m \ d )$\\[0.35em]
\hline
\textbf{AE} &\cmark&\xmark&\cmark&\xmark&\cmark&\cmark&\cmark&$\mathcal{O}((|\sV| + |\sF|) \ d )$&$\mathcal{O}(|\sV| \ m \ d )$\\[0.35em]
\textbf{MUSAE} &\cmark&\cmark&\cmark&\xmark&\cmark&\cmark&\cmark&$\mathcal{O}((|\sV| + |\sF|) \ d )$&$\mathcal{O}(|\sV| \ m \ d )$\\[0.35em]
\textbf{AE-EGO} &\cmark&\xmark&\cmark&\cmark&\cmark&\xmark&\cmark&$\mathcal{O}((|\sV| + |\sF|) \ d )$&$\mathcal{O}(|\sV| \ m \ d )$\\[0.35em]
\textbf{MUSAE-EGO} &\cmark&\cmark&\cmark&\cmark&\cmark&\xmark&\cmark&$\mathcal{O}((|\sV| + |\sF|) \ d )$&$\mathcal{O}(|\sV| \ m \ d )$\\[0.35em]
\hline
\end{tabular}

}
\end{table}

\section{Related work}\label{sec:related_work}
Efficient unsupervised learning of node embeddings for large networks has recently seen unprecedented development. The current paradigm focuses on learning latent node representations such that those sharing neighbors \cite{deepwalk,node2vec,rozemberczki2018fastsequence,gemsec}, structural roles \cite{refex, rolx, struc2vec, ahmed2018learning} or attributes \cite{yang2015network, yang2018binarized,feather,zhang2018sine} are close in the latent space. Our work falls under the first and last of these categories.
\subsection{A general overview} Several recent \textit{proximity-preserving} node embedding algorithms are inspired by the \textit{Skip-gram} model \cite{mikolov_1,mikolov_2}, which generates word embeddings by implicitly factorizing a matrix of shifted pointwise mutual information (PMI) of word co-occurrence statistics extracted from a text corpus \cite{levy2014neural}. 
For example, \textit{DeepWalk} \cite{deepwalk} generates a ``corpus'' of truncated random walks over a graph from which the \textit{Skip-gram} model generates proximity-preserving node embeddings. 
In doing so, \textit{DeepWalk} implicitly factorizes a PMI matrix that has been shown to correspond to the mean of a set of normalized adjacency matrix powers (up to a given order) reflecting different path lengths of a first-order Markov process \cite{qiu2018network}. 
Such averaging, or \textit{pooling} treats neighbors at different path lengths (or \textit{scales}) equally or according to fixed weightings \cite{mikolov_1, node2vec}; whereas it has been found that an optimal weighting may be task or dataset specific \cite{abu2018watch}. 
In contrast, \textit{multi-scale} node embedding methods, such as \textit{LINE} \cite{tang2015line}, \textit{GraRep} \cite{cao2015grarep} and \textit{Walklets} \cite{perozzidontwalk}, learn separate lower-dimensional embedding \textit{components} for each path length and concatenate them to form the full node representation. Such un-pooled representations, comprising distinct but less information at each scale, are found to give higher performance in a number of downstream settings, without increasing the overall complexity or number of free parameters \cite{tang2015line,cao2015grarep, perozzidontwalk}.

\textit{Attributed} node embedding methods refine ideas from proximity-preserving  node embeddings to also incorporate node \textit{attributes} (equivalently, features or labels) \cite{yang2015network, liao2018attributed, huang2017accelerated, yang2018binarized, yang2018enhanced}. Similarities between both a node's neighborhood structure and features contribute to determining pairwise proximity in the latent space, although models follow quite different strategies to learn such representations. The (arguably) simplest model, \textit{TADW} \cite{yang2015network}, decomposes a convex combination of normalized adjacency matrix powers into a matrix product that includes the feature matrix. Several other models, such as \textit{SINE} \cite{zhang2018sine} and \textit{ASNE} \cite{liao2018attributed}, implicitly factorize a matrix formed by concatenating the feature and adjacency matrices. Other approaches such as \textit{TENE} \cite{yang2018enhanced}, formulate the attributed node embedding task as a joint non-negative matrix factorization problem in which node representations obtained from sub-tasks are used to regularize one another. \textit{AANE} \cite{huang2017accelerated} uses a similar network structure based regularization approach, in which a node feature similarity matrix is decomposed using the alternating direction method of multipliers. \textit{BANE} \cite{yang2018binarized}, the method most similar to our own, learns attributed node embeddings that explicitly factorize the product of a normalized adjacency matrix power and a feature matrix. Many other methods exist, but do not consider the attributes of higher order neighborhoods~\cite{liao2018attributed, huang2017accelerated}.
A key difference between our work and previous methods is that we jointly learn \textit{distinct} representations of nodes and features. 

The relationship between our pooled (\textit{AE}) and multi-scale (\textit{MUSAE}) attributed node embedding methods mirrors that between graph convolutional neural networks (GCNNs) and  multi-scale GCNNs. The former, e.g. \textit{GCN} \cite{kipf2017semi}, \textit{GraphSage} \cite{graphsage_nips17}, \textit{GAT} \cite{gat_iclr18}, \textit{APPNP} \cite{klicpera_predict_2019}, \textit{SGCONV} \cite{sgc_icml19} and \textit{ClusterGCN} \cite{ clustergcn_kdd19}, create latent node representations that pool node attributes from arbitrary order neighborhoods, which are then inseparable and unrecoverable. The latter, e.g. \textit{MixHop} \cite{mixhop_icml19} and \textit{SIGN} \cite{sign} 
learn latent features for each proximity.
\subsection{A desiderata based comparison}
A node representation learning technique must have certain desired properties in order to generate expressive vertex features and have scalability. We summarized these beneficial properties of node embedding techniques in Table \ref{fig:musae_comparison} with the respective space and time complexities. 

\begin{itemize}
    \item \textbf{Generic features:} Generic vertex properties such as the income of users in a social network are encoded when a node embedding is learned. These generic features are used to contextualize the location of the node in the embedding space universally.
    \item \textbf{Multi-scale:} Information obtained from distinct proximities (e.g. random walk hops, shortest path distance) is encoded by distinct groups of node embedding features. Using a multi-scale node representation the micro-, meso-, and macroscopic context of a node can be discerned.  
    \item \textbf{Implicit:} The decomposed target matrix is not calculated explicitly. This reduces the required space and time complexity, which makes the embedding model applicable in practical large-scale industrial settings.
    \item \textbf{Proximal:} The contextual proximity information (location in comparison to other vertices) about the node is encoded when the node embedding is created. A proximal node embedding cannot be inductive, as the proximal context is not be meaningful in disjoint graphs.
    \item \textbf{Higher-order:} The embedding encodes information from nodes that are not adjacent to a node. For example, in random walk based contextualization the information is obtained from multiple hops, not just the first step.
    \item \textbf{Inductive:} The node embedding technique can map unseen nodes to the embedding space which are not connected to the graph used for training. An embedding technique which contextualizes the nodes based on the proximity to other nodes cannot be inductive.
    \item \textbf{Non-linear:} A node-node proximity score in the target matrix is not a linear function of the two node embedding vectors. This property allows for super and sub linear proximity score encoding. 
\end{itemize}
\vspace{-5mm}

\begin{algorithm2e}[th!]
    
    \vspace{5pt}
    \DontPrintSemicolon
    \SetAlgoLined
    \footnotesize
    \KwData{$\gG =(\sV, \mathbb{E})$ -- Graph to be embedded.\\
    \quad \quad\,\,\,\,\,$\{\sF_v\}_{\sV}$ -- Set of node feature sets.\\
    \quad \quad\,\,\,\,\,$s$ -- Number of sequence samples.\\
    \quad \quad\,\,\,\,\,$l$ -- Length of sequences.\\
    \quad \quad\,\,\,\,\,$t$ -- Context size.\\
    \quad \quad\,\,\,\,\,$d$ -- Embedding dimension.\\
    \quad \quad\,\,\,\,\,$b$ -- Number of negative samples.\\
    }
    \vspace{5pt}
    \KwResult{Node embedding $g$ and feature embedding $h$.}
    \vspace{5pt}
    
    \For{i\,\,\text{\upshape in}\,\,$1:s$}{
    
        Pick $v_1\in \sV$ according to $P(v) \sim deg(v)/\text{vol}(\gG).$\;
        
    	$(v_1,v_2,\dots,v_l)\leftarrow$ Sample Nodes$(\gG,v_1, l)$\;
    	
        \For{$j \,\, \textup{in} \,\, 1:l-t$}{
    	        \For{$r \,\, \text{\upshape in} \,\, 1:t$}{
    	        \For{$f \,\, \text{\upshape in} \,\, \sF_{v_{j+r}}$}{
    	        	Add tuple $(v_j,f)$ to multiset $\sD$.\;
    	        	
    	        }
    	        \For{$f \,\, \text{\upshape in} \,\, \sF_{v_{j}}$}{
    	            Add tuple $(v_{j+r},f)$ to multiset $\sD$.\;
    	        }	    
    	      }
         }

    }
    
    \vspace{5pt}
    Run SGNS on $\sD$ with $b$ negative samples and $d$ dimensions.\;
    
    Output $g_v,\,\, \forall v \in \sV$, and $h_f,\,\, \forall f \in \gF=\cup_\sV \sF_v$.  
    \vspace{5pt}
    {\small\caption{AE sampling and training procedure}\label{AE_algo}}
    
\end{algorithm2e}
\vspace{-8mm}
\section{Attributed embedding algorithms}\label{sec:model}

Here we define our algorithms
to jointly learn embeddings of nodes and attributes based on the structure and attributes of local neighborhoods. 
The aim is to learn similar embeddings for nodes that occur in neighborhoods of similar attributes; and similar embeddings for attributes that occur in similar neighborhoods of nodes.

Let $\gG \!=\! (\sV,\mathbb{E})$ be an undirected graph, where $\sV$ and $\mathbb{E}$ are the sets of vertices and edges (or links), and let $\sF$ be the set of all binary node features. 
For each node $v \!\in\! \sV$, let $\sF_v \!\subseteq\! \sF$ be the subset of features belonging to $v$. An embedding of nodes is a mapping $g:\sV \!\to\! \sR^d$ that assigns a $d$-dimensional vector $g(v)$, or simply $g_v$, to each node $v$ and is fully described by a matrix $\mG \!\in\! \sR^{|\sV| \times d}$. An embedding of the features (to the same latent space) is a mapping $h: \sF\to \smash{\sR^{d}}$ with embeddings denoted $h(f) \!\doteq\! h_f$, as summarised by a matrix $\mH \!\in\! \sR^{|\sF| \times d}$. 
\subsection{Attributed embedding}

The \textit{Attributed Embedding} (\textit{AE}) method is described by Algorithm \ref{AE_algo} and the main idea is figuratively summarized in Figure \ref{fig:ae}. With probability proportional to node degree, $s$ starting nodes $v_1$ are sampled from $\sV$  (line 2). From each starting node, a node sequence of length $l$ is sampled over $\gG$ following a first-order random walk (line 3). For a given window size $t$, iterate over the first $l \!-\! t$ nodes of the sequence $v_j$, termed \textit{source} nodes (line 4). For each source node, the subsequent $t$ nodes are considered \textit{target} nodes  (line 5). For each target node $v_{j+r}$, the tuple $(v_j,f)$ is added to the corpus $\sD$ for each feature $f \!\in\! \sF_{v_{j+r}}$ (lines 6-7). Each tuple $(v_{j+r},f)$ for features of the source node $f \!\in\! \sF_{v_{j}}$ is also added to $\sD$ (lines 9-10). Running Skip-gram on $\sD$ with $b$ negative samples (line 15) generates the $d$-dimensional node and feature embeddings ($\mG$, $\mH$). 

\input{figures/musae_illustrative}
\subsection{Multi-scale attributed embedding}

The \textit{AE} algorithm pools features across neighborhoods of different proximity. Inspired by the performance of (unattributed) multi-scale node embeddings, we adapt \textit{AE} to learn \textit{multi-scale} attributed embeddings. The procedure is described by Algorithm \ref{MUSAE_algo} and the main idea is figuratively summarized in Figure \ref{fig:musae}.
The embedding component of a node $v\!\in\! \sV$ at proximity $r\!\in\!\{1,... ,t\}$ is given by a mapping $g^r:\sV\to \sR^{d/t}$ (where $t$ divides $d$). Similarly, the embedding component of feature $f \!\in\! \sF$ at proximity $r$ is given by a mapping $h^r:\sF\to \sR^{d/t}$. Concatenating the $t$ components gives a $d$-dimensional embedding for each node and feature.
The \textit{Multi-Scale Attributed Embedding} (\textit{MUSAE}) method is described by Algorithm \ref{MUSAE_algo}. 
Source and target node pairs are generated from sampled node sequences as for \textit{AE} (lines 2-5). 
Each feature of a target node $f \!\in\!  \sF_{v_{j+r}}$ is again considered, but tuples $(v_j,f)$ are added to a \textit{sub-corpus} $\sD_{\overset{r}{\rightarrow}}$ (lines 6-7) and for each source node feature $\smash{f \!\in\! \sF_{v_j}}$ tuples $(v_{j+r},f)$ are added to another sub-corpus $\sD_{\overset{r}{\leftarrow}}$  (lines 9-10). 
Running Skip-gram with $b$ negative samples on each sub-corpus \smash{$\sD_{r} \!=\! \sD_{\overset{r}{\rightarrow}} \cup \sD_{\overset{r}{\leftarrow}}$}  (line 17) generates the \smash{$\tfrac{d}{t}$}-dimensional components to concatenate. 

    \begin{algorithm2e}[h!]
    \DontPrintSemicolon
    \SetAlgoLined

    \setcounter{AlgoLine}{0}
    \vspace{5pt}
    \footnotesize
    \KwData{$\gG =(\sV,\mathbb{E})$ -- Graph to be embedded.\\
    \quad \quad\,\,\,\,\,$\{\sF_v\}_{\sV}$ -- Set of node feature sets.\\
    \quad \quad\,\,\,\,\,$s$ -- Number of sequence samples.\\
    \quad \quad\,\,\,\,\,$l$ -- Length of sequences.\\
    \quad \quad\,\,\,\,\,$t$ -- Context size.\\
    \quad \quad\,\,\,\,\,$d$ -- Embedding dimension.\\
    \quad \quad\,\,\,\,\,$b$ -- Number of negative samples.\\
    }
    \vspace{5pt}
    \KwResult{Node embedding components $g^r$ and feature embeddings component $h^r$, for $r \!\in\!\{1,...,t\}$.}
    \vspace{5pt}
    \For{i\,\,\textup{in}\,\,$1:s$}{
        Pick $v_1\in \sV$ according to $P(v) \sim deg(v)/\text{vol}(\gG) .$\;
        
        $(v_1,v_2,\dots,v_l)\leftarrow$ Sample Nodes$(\gG,v_1, l)$\;
        
        \For{$j \,\, \textup{in} \,\, 1:l-t$}{
    	        \For{$r \,\, \textup{in} \,\, 1:t$}{
    	        \For{$f \,\, \textup{in} \,\, \sF_{v_{j+r}}$}{
    	        	Add the tuple $(v_j,f)$ to multiset $\sD_{\overset{r}{\rightarrow}}$.\;
    	        }
    	        \For{$f \,\, \textup{in} \,\, \sF_{v_{j}}$}{
    	        Add the tuple $(v_{j+r},f)$ to multiset $\sD_{\overset{r}{\leftarrow}}$.\;
    	        }	    
    	      }
         }

    }
    \For{$r \,\, \textup{in} \,\, 1:t$}{
        Create $\sD_r$ by unification of \smash{$\sD_{\overset{r}{\rightarrow}}$ and $\sD_{\overset{r}{\leftarrow}}$}.\;

        Run SGNS on $\sD_r$ with $b$ negative samples and $\tfrac{d}{t}$ dimensions.\;

        Output $g^r_v,\,\, \forall v \in \sV$, and $h^r_f,\,\, \forall f \in \sF = \cup_\sV \sF_v$.  
    }
    {\small\caption{MUSAE sampling and training procedure}\label{MUSAE_algo}}
    \end{algorithm2e}

\section{Attributed embedding as implicit matrix factorization}\label{sec:theory}

The results of \cite{levy2014neural} showed that the loss function of Skip-gram with negative sampling (SGNS) is minimized if its two output embedding matrices $\mW, \mC$ factorize a matrix of pointwise mutual information (PMI) of word co-occurrence statistics. Specifically, for a corpus $\sD$ over a dictionary $\sV$ with $|\sV| \!=\! n$, SGNS (with $b$ negative samples)  generates embeddings $\vw_w$, $\vc_c\!\in\!\sR^d$ (columns of $\mW$, $\mC \!\in\! \sR^{d\times n}$) for each target and context word $w, c \!\in\! \sV$, satisfying:
$$
\vw_w^\top \vc_c^{\phantom\top} 
\approx
\log \big( \tfrac{\#(w,c)|\sD|}{\#(w)\#(c)}   \big) - \log b\ ,
$$
where $\#(w)$, $\#(c)$ and $\#(w,c)$ denote counts of $w$, $c$ and both words appearing within a sliding context window. 
Considering $\tfrac{\#(w)}{|\sD|}$,  $\tfrac{\#(c)}{|\sD|}$,  $\tfrac{\#(w,c)}{|\sD|}$ as empirical estimates of $p(w)$, $p(c)$ and $p(w,c)$ respectively gives the approximate low-rank factorization (since typically $d\!\ll\!n$):
$$
\mW^\top \mC 
\approx
\left[\,\textup{PMI}(w,c) - \log b \,\right]_{w,c\in\sV}\, ,
$$
The findings of \cite{qiu2018network} extended this result to node embedding models that apply SGNS to a ``corpus'' generated from random walks over the graph. In the case of \textit{DeepWalk} where random walks are first-order Markov, the joint probability distributions over nodes at different steps of a random walk can be expressed in closed form, and a closed form for the factorized PMI matrix follows. Here we derive the matrices implicitly factorized by \textit{AE} and \textit{MUSAE}. 

\keypoint{Notation:} For a graph $\gG \!=\! (\sV, \mathbb{E})$, $|\sV| \!=\! n$, let $\mA \!\in\!\sR^{n\times n}$ denote the adjacency matrix and $\mD \!\in\!\sR^{n\times n}$ the diagonal degree matrix, i.e.
$
\mD_{v,v} \!=\! \text{deg}(v) \!=\! \sum_w\mA_{v,w}$.
Let $c=\sum_{v,w}\mA_{v,w}$ denote the \textit{volume} of $\gG$. We define the binary attribute matrix $\mF\!\in\!\{0,1\}^{|\sV| \times | \sF|}$ by
$
\mF_{v,f}=\textbf{1}
_{f \in \sF_v},
\ \forall v \!\in\! \sV, f \!\in\!  \sF.
$
To ease notation, we set $\mP \!=\! \mD^{-1}\mA$ and $\mE \!=\! diag(\1^\top \mD\mF)$, where $diag$ indicates a diagonal matrix.

\keypoint{Interpretation:} Assuming $\gG$ is ergodic, $p(v) \!=\! \tfrac{deg(v)}{c}$ is the stationary distribution over nodes $v\!\in \!\sV$, i.e. $c^{-1}\mD=diag(p(v))$; and $c^{-1}\mA$ is the stationary joint distribution over consecutive nodes of a random walk $p(v_j,v_{j+1})$. 
$\mF_{v, f}$ can be considered a Bernoulli parameter of the probability $p(f|v)$ of observing feature $f$ at a node $v$, hence $c^{-1}\mD\mF$ describes the stationary joint distribution $p(f,v)$ over nodes and features.
Accordingly, $\mP$ is the transition matrix of conditional probabilities $p(v_{j+1}|v_j)$; and $\mE$ is a diagonal matrix proportional to the (binary) probability $p(f)$ of observing feature $f$ at the stationary distribution. We note that $p(f)$ need not sum to 1, whereas $p(v)$ necessarily must.


\subsection{Multi-scale case (MUSAE)}

We know that the SGNS aspect of \textit{MUSAE} learns embeddings $g^r_v$, $h^r_f$ that satisfy
$
{g^r_v}^\top h^r_f
 \!\approx\! 
\log\big(\tfrac{\#(v,f)_r|\sD_r|}{\#(v)_r\# (f)_r}\big)-\log b$,
$\forall v \!\in\! \sV, f \!\in\! \sF.
$
Our aim is to express this factorization in terms of known properties of the graph $\gG$ and its features.

\vspace{10pt}

\begin{restatable}{lemma}{lemmaplims}
The empirical statistics of node-feature pairs obtained from random walks give unbiased estimates of the joint probability of observing feature $f\!\in \!\sF$ $r$ steps \textbf{(i) after; or (ii) before} node $v\in \sV$, as given by:
\begin{align*}
    \underset{l\rightarrow \infty}{\text{plim}}
    \tfrac{\# (v,f)_{\underset{r}{\rightarrow}}}{|\sD_{\underset{r}{\rightarrow}}|}
    &= c^{-1}(\mD\mP^r\mF)_{v,f}
    \\
    \underset{l\rightarrow \infty}{\text{plim}}\tfrac{\# (v,f)_{\underset{r}{\leftarrow}}}{|\sD_{\underset{r}{\leftarrow}}|}
    &= c^{-1}(\mF^{\top}\mD\mP^r)_{f,v}
\end{align*}
\end{restatable}
\begin{proof}
    See Appendix \ref{app:proofs}.
\end{proof}


\begin{restatable}{lemma}{lemmaplim}\label{lem:MUSAE}
The empirical statistics of node-feature pairs obtained from random walks give unbiased estimates of the joint probability of observing feature $f\!\in \!\sF$ $r$ steps \textbf{either side} of node $v \!\in\! \sV$, given by:
\begin{align*}
    \underset{l\rightarrow \infty}{\text{plim}}\tfrac{\# (v,f)_r}{|\sD_r|}&=
    c^{-1}(\mD\mP^r \mF)_{v,f}
    \ ,
\end{align*}
\end{restatable}
\begin{proof}
    See Appendix \ref{app:proofs}.
\end{proof}
\noindent Marginalizing gives unbiased estimates of stationary probability distributions of nodes and features:
\begin{align*}
\underset{l\rightarrow \infty}{plim}\tfrac{\# (v)}{|\sD_r|}
&
= \tfrac{deg(v)}{c}
=  c^{-1}\mD_{v,v}
\\
\underset{l\rightarrow \infty}{plim}\tfrac{\# (f)}{|\sD_r|}
&
= \sum_{v|f \in \sF_v} \! \tfrac{deg(v)}{c}
= c^{-1}\mE_{f,f}
\end{align*}

\begin{theorem}\label{thm:MUSAE}
MUSAE embeddings approximately factorize the node-feature PMI matrix:
$$	
log\left(
c\,\mP^r\mF\mE^{-1}
\right)-\log b,\quad \text{for}\ r=1, ...\,, t .
$$
\end{theorem}

\begin{proof}

\begin{align*}
    \tfrac{\# (v,f)_r |\sD_r|}{\# (f)_r \# (v)_r}
    & = 
    \big( \tfrac{\# (v,f)_r}{|\sD_r|} \big) /
    \big( \tfrac{\# (f)_r}{|\sD_r|} \tfrac{\# (v)_r}{|\sD_r|}\big)
    \\
    & \xrightarrow{p}
    %
%
    \big((c \mD^{-1})(c^{-1}\mD\mP^r \mF)(c\mE^{-1} )\big)_{v,f}
    \\
    & =
    c ( \mP^r \mF\mE^{-1} )_{v,f}
    \qedhere
\end{align*}
\end{proof}
\color{black}

\subsection{Pooled case (AE)}

\begin{lemma}
The empirical statistics of node-feature pairs learned by the AE algorithm give unbiased estimates of mean joint probabilities over different path lengths:
\begin{equation*}
\underset{l\rightarrow \infty}{\text{plim}}\tfrac{\# (v,f)}{|\sD|}
=
\tfrac{c}{t}\big( \mD(\sum_{r=1}^t \mP^r) \mF\big)_{v,f}
\end{equation*}
\end{lemma}
\begin{proof}
By construction, 
$|\sD| \! =\! \sum_r|\sD_r|$, 
$\#(v,f)\! =\! \sum_r\#(v,f)_r$, 
$|\sD_r| \!=\! |\sD_s|\ \forall\ r,s \!\in\! \{1,..., t\}$ and so $|\sD_s| \!=\! t^{-1}|\sD|$. Combining with Lemma \ref{lem:MUSAE}, the result follows.
\end{proof}
\begin{theorem}\label{thm:AE}
AE embeddings approximately factorize the pooled node-feature matrix:
$$\log\big(\tfrac{c}{t}(\sum_{r=1}^t \textbf{P}^r )\textbf{F}\textbf{E}^{-1}\big)-\log b\ .
$$
\end{theorem}
\begin{proof}
Analogous to the proof of Theorem \ref{thm:MUSAE}.
\end{proof}

\begin{remark}
\textit{DeepWalk} is a special case of \textit{AE} with $\mF \!=\! \mI_{|\sV|}$.
\end{remark}
That is, \textit{DeepWalk} is equivalent to \textit{AE} if each node has a single unique feature. Thus $\mE \!=\! diag(\1^\top\mD\mI) \!=\! \mD$ and, by Theorem \ref{thm:AE}, the embeddings of \textit{DeepWalk}  factorize
$\log\big(\tfrac{c}{t}(\sum_{r=1}^t \mP^r )\mD^{-1}\big)-\log b$,
as derived by \cite{qiu2018network}.

\vspace{5pt}
\begin{remark}
\textit{Walklets} is a special case of \textit{MUSAE}  with $\mF \!=\! \mI_{|\sV|}$. 
\end{remark}
Thus, for $r=1,\dots,t$, the embeddings of \textit{Walklets} factorise
$
\log\left(c\,\mP^r \mD^{-1}\right)-\log b.
$

\vspace{5pt}
\begin{remark}\label{remark:MUSAE_EGO}
Appending an identity matrix $\mI$ to the feature matrices $\mF$ of AE and MUSAE (denoted $\left [\mF;\mI\right ]$) adds a unique feature to each node. The resulting algorithms, named \textit{AE-EGO} and \textit{MUSAE-EGO}, respectively, learn embeddings that approximately factorize the node-feature PMI matrices:
\begin{align*}
&\log\left(
c\,\mP^r\left [\mF;\mI\right ]\mE^{-1} \right)-\log b,\ \ \forall r\!\in\!\{1, ...,t\};
%
\\
\text{and}
\qquad
%
&\log\big(\tfrac{c}{t}(\sum_{r=1}^t \mP^r )\left [\mF;\mI\right ]\mE^{-1}\big)-\log b\ .
\end{align*}	
\end{remark}

\subsection{Complexity analysis}\label{sec:complexity}
 Assuming a constant number of features per source node, the corpus generation has runtime complexity of $\mathcal{O}(s\, l\,  t \, \tfrac{m}{n})$, where $m \!=\! \sum_{v\in \sV} | \sF_v|$ the total number of features across all nodes (with repetition), $q \!=\! |\sF|$, and $n \!=\! |\sV|$. With $b$ negative samples, the optimization runtime of a single asynchronous gradient descent epoch on \textit{AE} and the joint optimization runtime of \textit{MUSAE} embeddings is $\mathcal{O}( b \, d \, s \, l \, t \, \tfrac{m}{n})$. With $p$ truncated walks from each source node, the corpus generation complexity is $\mathcal{O}(p\,n \, l \, t \, m)$ and the model optimization runtime is $\mathcal{O}( b \, d\,  p\, n \, l \, t \, m)$. The runtime experiments (Section \ref{sec:experiments}) empirically support this analysis. 

Corpus generation has a memory complexity of $\mathcal{O}(s \, l \, t \, \tfrac{m}{n})$ while the same when generating $p$ truncated walks per node has a memory complexity of $\mathcal{O}(p\, n \, l \, t \, m)$. Storing the parameters of an \textit{AE} embedding has a memory complexity of $\mathcal{O}((n+q)\,\cdot d)$ and \textit{MUSAE} uses $\mathcal{O}((n+q)\,\cdot d )$ memory.

\begin{table}[h!]
\setlength{\tabcolsep}{7pt}

\caption{Descriptive statistics of attributed benchmark social network and webgraph datasets.}\label{tab:descriptive_statistics}
\centering{\scriptsize

\begin{tabular}{lrrrrrrcc}
\textbf{Dataset} & \textbf{Nodes} & \specialcell{\textbf{Clustering}\\\textbf{Coefficient}} & \textbf{Density} &\textbf{Diameter} &\specialcell{\textbf{Unique}\\\textbf{Features}}& \specialcell{\textbf{Features}\\\textbf{Per Node}}&\textbf{Classes}&\textbf{Task}\\
\hline
Cora & 2,708&0.094&0.002&19&1,432&18.174&7&Classification\\
Citeseer&3,327&0.130& 0.001&28&3.703&31.610&6&Classification\\
Pubmed  & 19,717&0.054&0.001&18&500&50.511&3&Classification\\
\hline
Facebook&22,470     &0.232\   &0.001&15&4,714&14.000&4&Classification\\[0.25em]
GitHub&37,700     &0.013\   &0.001&7&4,005&18.312&2&Classification\\[0.25em]
LastFM Asia&   7,624 &0.179&0.001&15&7,842&395.378&18&Classification\\[0.25em]
Deezer Europe& 28,281&0.096&0.001&21&30978&33.891&2&Classification\\[0.25em]
\hline
Wiki Chameleons &2,277      &0.314\   &0.012&11&3,132&21.545&--&Regression\\[0.25em]
Wiki Crocodiles &11,631     &0.026  &0.003&11&13,183&75.161&--&Regression\\[0.25em]
Wiki Squirrels  &5,201      &0.348  &0.015&10&3,148&26.474&--&Regression\\[0.25em]
\hline 
Twitch DE       &9,498      &0.047  &0.003&7&3,169&20.396&2&Classification\\[0.25em]
Twitch EN       &7,126      &0.042  &0.002&10&3,169&20.800&2&Classification\\[0.25em]
Twitch ES       &4,648      &0.084  &0.006&9&3,169&19.391&2&Classification\\[0.25em]
Twitch FR       &6,549      &0.054  &0.005&7&3,169&19.757&2&Classification\\[0.25em]
Twitch PT       &1,912      &0.131  &0.017&7&3,169&19.944&2&Classification\\[0.25em]
Twitch RU       &4,385      &0.049  &0.004&9&3,169&20.635&2&Classification\\[0.25em]
\hline
\end{tabular}
}
\end{table}

\vspace{-7mm}
\section{Experimental evaluation}\label{sec:experiments}

We evaluate the representations learned by \textit{AE}, \textit{MUSAE} and their \textit{EGO} extensions on several common downstream tasks: node classification, regression, and transfer learning across networks. We also report how number of nodes and dimensionality affect runtime. We use standard (well-established) node classification webgraph benchmark datasets (Cora, Citeseer \cite{lu2003link}, Pubmed \cite{namata2012query}) together with publicly available social network benchmarks -- LastFM Asia, Deezer Europe \cite{feather}). We also utilized social networks and web graphs that we collected (e.g Twitch, Facebook, Github, Wikipedia). Table \ref{tab:descriptive_statistics} shows statistics of the datasets used for evaluation.
\begin{itemize}
\item \textbf{Facebook:}
 A page-page graph of verified Facebook sites. Nodes correspond to official Facebook pages, links to mutual likes between sites. Node features are extracted from the site descriptions. The task is multi-class classification of the site category.
\item \textbf{GitHub:}
A social network where nodes correspond to developers who have starred at least 10 repositories and edges to mutual follower relationships. Node features are location, starred repositories, employer and e-mail address. The task is to classify nodes as web or machine learning developers.

\item \textbf{LastFM Asia:} An online social network of people who use the online music streaming site LastFM and live in Asia. The links represent reciprocal follower relationships and the vertex features describe the list of musicians liked by the users. The machine learning task is the prediction of nationality for the users of the site.
\item \textbf{Deezer Europe:} A user-user network of European members of the music streaming service Deezer. The links represent mutual friendships of the users. Node features are artists liked by the streamers and the related task is the classification of the users' gender.
\item \textbf{Wikipedia graphs:}
Wikipedia page-page networks on three topics: chameleons, crocodiles and squirrels. Nodes represent articles from the English Wikipedia (December 2018), edges reflect mutual links between them. Node features indicate the presence of particular nouns in the articles and the average monthly traffic (October 2017 - November 2018). The regression task is to predict the log average monthly traffic (December 2018).
\item \textbf{Twitch social networks:}
User-user networks where nodes correspond to Twitch users and links to mutual friendships. Node features are games liked, location and streaming habits. All datasets have the same set of node features enabling \textit{transfer learning} across networks. The associated task is binary classification of whether a streamer uses explicit language.
\end{itemize}
\vspace{-5mm}
\begin{table}[ht!]
\centering
\caption{Hyperparameters of ode embedding techniques and the supervised downstream models (classifiers and regressors).}
\vspace{2mm}
	\centering
	\footnotesize
    \begin{tabular}{lcc}
        \label{tab:our_params}
        \textbf{Parameter}          &\textbf{Notation}  &\textbf{Value}\\
        \hline
        Dimensions                  & $d$               & 128   \\
        Walk length                 & $l$               & 80    \\
        Walks per node              & $p$               & 10    \\
        Epochs                      & $e$               & 5     \\
        Window size                 & $t$               & 3     \\
        Negative samples            & $b$               & 5     \\
        Initial learning rate       &  $\alpha_{\text{max}}$   & 0.05  \\
        Final learning rate         & $\alpha_{\text{min}}$    & 0.025 \\
        \hline
        Regularization coefficient          &$\lambda$      & 0.01 \\
        Norm mixing parameter               &$\gamma$       & 0.5   \\
        \hline
    \end{tabular}

\end{table}

\subsection{Hyperparameter settings}\label{subsec:hyperparams}

Table \ref{tab:our_params} (top) shows the hyperparameters used for our algorithms, which are consistent with other random walk based approaches \cite{deepwalk, node2vec, struc2vec, perozzidontwalk}. Hyperparameters of other algorithms are in Appendices \ref{app:embedding} and \ref{app:gcn}. Downstream evaluation tasks use logistic and elastic net regression from \textit{Scikit-learn} \cite{pedregosa2011scikit} with standard library settings except for parameters reported in  Table \ref{tab:our_params} (bottom). 
\vspace{-4mm}
\begin{table}[!b]

\caption{Node classification: average micro $F_1$ on the test set score and standard error over 100 seeded runs (best \textit{unsupervised} result in red, best \textit{supervised} result in blue, our proposed embedding methods in italics).}\label{tab:classification_evaluation}
\centering{\footnotesize
\setlength{\tabcolsep}{4pt}
    \begin{tabular}{lcccccccc}
     \textbf{} & \textbf{Facebook} &\textbf{GitHub}&\textbf{Twitch PT}& \textbf{LastFM}  & \textbf{Deezer}& \textbf{Cora}  & \textbf{Citeseer} & \textbf{PubMed} \\     \hline
\text{DeepWalk}&$.863\pm.001$&$.858\pm.001$&$.672\pm.007$&$.765\pm.003$&$.556\pm.001$&$.833\pm.004$&$.603\pm.007$&$.802\pm.001$\\[0.5em]
\text{LINE}$_2$&$.875\pm.001$&$.858\pm.001$&$.670\pm.001$&$.842\pm.005$&$.565\pm.002$&$.777\pm.004$&$.542\pm.006$&$.799\pm.001$\\[0.5em]
\text{Node2Vec}&$.890\pm.001$&$.859\pm.001$&$.686\pm.004$&$.791\pm.002$&$.563\pm.001$&$.840\pm.003$&$.622\pm.005$&$.810\pm.002$\\[0.5em]
\text{Walklets}&$.887\pm.001$&$.860\pm.001$&$.671\pm.006$&$.849\pm.002$&$.562\pm.001$&$.843\pm.003$&$.630\pm.006$&$.815\pm.001$\\[0.5em]
\text{NetMF}&$.795\pm.005$&$.839\pm.001$&$.647\pm.002$&$.826\pm.001$&$.561\pm.001$&$.748\pm.002$&$.616\pm.002$&$.773\pm.005$\\[0.5em]
\text{HOPE}&$.702\pm.006$&$.800\pm.001$&$.623\pm.001$&$.764\pm.002$&$.562\pm.001$&$.716\pm.002$&$.583\pm.002$&$.705\pm.003$\\[0.5em]
\text{GraRep}&$.878\pm.002$&$.854\pm.003$&$.614\pm.001$&$.813\pm.001$&$.563\pm.002$&$.732\pm.002$&$.637\pm.002$&$.784\pm.004$\\[0.5em]\hline
\text{TADW}&$.765\pm.002$&$.748\pm.001$&$.659\pm.005$&$.586\pm.003$&$.636\pm.002$&$.819\pm.004$&$.734\pm.004$&$.862\pm.002$\\[0.5em]
\text{AANE}&$.796\pm.001$&$.856\pm.001$&$.661\pm.006$&$.750\pm.002$&$622\pm.001$&$.793\pm.001$&$.733\pm.004$&$\color{red}\mathbf{.867\pm.001}$\\[0.5em]
\text{ASNE}&$.797\pm.001$&$.839\pm.001$&$\color{red}\mathbf{.685\pm.006}$&$.789\pm.002$&$.625\pm.001$&$.830\pm.003$&$.718\pm.004$&$.846\pm.002$\\[0.5em]

\text{BANE}&$.868\pm.001$&$.762\pm.001$&$.664\pm.006$&$.581\pm.005$&$.558\pm.001$&$.807\pm.005$&$.713\pm.003$&$.823\pm.002$\\[0.5em]
\text{TENE}&$.731\pm.002$&$.850\pm.001$&$.664\pm.006$&$.679\pm.002$&$.622\pm.003$&$.829\pm.005$&$.681\pm.003$&$.842\pm.001$\\[0.5em]
\hline

\textit{AE}&$.888\pm.001$&$.863\pm.001$&$.672\pm.004$&$.866\pm.001$&$.658\pm.001$&$.835\pm.005$&$.739\pm.005$&$.839\pm.002$\\[0.5em]
\textit{AE-EGO}&$\color{red}\mathbf{.899\pm.001}$&$.863\pm.001$&$.671\pm.007$&$.868\pm.001$&$\color{red}\mathbf{.662\pm.001}$&$.835\pm.006$&$.739\pm.005$&$.840\pm.003$\\[0.5em]
\textit{MUSAE}&$.887\pm.001$&$\color{red}\mathbf{.864\pm.001}$&$.672\pm.006$&$\color{red}\mathbf{.870\pm.003}$&$\color{red}\mathbf{.662\pm.002}$&$.848\pm.004$&$\color{red}\mathbf{.742\pm.004}$&$.853\pm.001$\\[0.5em]
\textit{MUSAE-EGO}&$.894\pm.001$&$\color{red}\mathbf{.864\pm.001}$&$.671\pm.002$&$.865\pm.002$&$.661\pm.001$&$\color{red}\mathbf{.849\pm.004}$&$.741\pm.004$&$.851\pm.002$\\[0.5em]
\hline
\text{GCN}&$.932\pm.001$&$.865\pm.001$&$.695\pm.007$&$.874\pm.001$&$.635\pm.002$&$.879\pm.001$&$.742\pm.001$&$.875\pm.001$\\[0.5em]
\text{GraphSAGE}&$.814\pm.002$&$.854\pm.001$&$.631\pm.004$&$.871\pm.001$&$.661\pm.002$&$.881\pm.001$&$.747\pm.001$&$.864\pm.001$\\[0.5em]
\text{GAT}&$.919\pm.001$&$.864\pm.001$&$.678\pm.001$&$.864\pm.001$&$.625\pm.003$&$.867\pm.002$&$.740\pm.001$&$.869\pm.001$\\[0.5em]
\text{MixHop}&$\color{blue} \mathbf{.941\pm.002}$&$.850\pm.001$&$.630\pm.004$&$\color{blue}\mathbf{.888\pm.003}$&$\color{blue}\mathbf{.667\pm.001}$&$.859\pm.001$&$\color{blue}\mathbf{.780\pm.001}$&$\color{blue}\mathbf{.891\pm.001}$\\[0.5em]
\text{ClusterGCN}&$.937\pm.001$&$.859\pm.001$&$.654\pm.001$&$.761\pm.002$&$.634\pm.002$&$.845\pm.001$&$.737\pm.001$&$.844\pm.001$\\[0.5em]
\text{APPNP}&$.938\pm.001$&$\color{blue} \mathbf{.868\pm.001}$&$\color{blue}\mathbf{.755\pm.001}$&$.845\pm.002$&$.622\pm.002$&$\color{blue}\mathbf{.888\pm.001}$&$.754\pm.001$&$.884\pm.001$\\[0.5em]
\text{SGCONV}&$.836\pm.002$&$.829\pm.001$&$.663\pm.003$&$.846\pm.001$&$.601\pm.001$&$.878\pm.002$&$.763\pm.002$&$.807\pm.001$\\[0.5em]
\hline
    \end{tabular}
    
}

\end{table}

\subsection{Node attribute classification}

Embedding algorithms take as input a network graph and node features. Classification performance is evaluated by training $l_2$-regularized logistic/softmax regression to predict a (test) attribute given a node embedding. Table \ref{tab:classification_evaluation} compares our models to leading node embedding methods by micro averaged $F_1$ score over 100 seeded splits (80\% train - 20\% test). 

Whilst relative model performance varies slightly with dataset, the results show that our attributed embeddings tend to outperform other unsupervised methods, with closest performance achieved by \textit{ASNE} and \textit{AANE}. We also observe that (i) multi-scale embeddings tend to outperform their pooled counterparts; (ii) the additional identity features of \textit{EGO} models have no material impact on the task; and (iii) attributed node embeddings that consider only first-order neighbours show weak performance. As an upper bound comparison, we report performance of \textit{supervised} methods, which shows a typical performance advantage of $\sim\!4\%$ that can be within 1\%. 
\begin{figure*}[t]
	\centering
	\begin{tikzpicture}[scale=0.45,transform shape]
	\tikzset{font={\fontsize{13pt}{12}\selectfont}}
	\begin{groupplot}[group style={group size=3 by 1,horizontal sep=50pt, vertical sep=70pt, ylabels at=edge left},
yticklabel style={/pgf/number format/.cd,precision=2, /pgf/number format/fixed, /pgf/number format/fixed zerofill},	
	grid=major,
	grid style={dashed, gray!40},
	scaled ticks=false,
	inner axis line style={-stealth}]
	\nextgroupplot[
	ytick={0.4,0.5,0.6,0.7,0.8},
	xtick={6,12,18,24,30},
	ymin=0.3,
	ymax=0.81,
	xmin=1,
	xmax=32,
	xlabel={Training samples per class},
	ylabel={Test micro $F_1$ score},
	title=Facebook,]
	
	\addplot[mark=triangle*,opacity=0.8,mark options={black,fill=yellow},mark size=5pt]
	coordinates {
(3,0.385)
(6,0.428)
(9,0.463)
(12,0.471)
(15,0.503)
(18,0.508)
(21,0.519)
(24,0.535)
(27,0.539)
(30,0.549)};

	\addplot[mark=diamond*,opacity=0.8,mark options={black,fill=brown},mark size=5pt]
	coordinates {
(3,0.376)
(6,0.42)
(9,0.462)
(12,0.504)
(15,0.511)
(18,0.539)
(21,0.552)
(24,0.575)
(27,0.597)
(30,0.607)};

	\addplot[mark=*,opacity=0.8,mark options={black,fill=pink},mark size=3pt]
	coordinates {
(3,0.343)
(6,0.379)
(9,0.423)
(12,0.464)
(15,0.474)
(18,0.489)
(21,0.507)
(24,0.524)
(27,0.527)
(30,0.536)};
	
	\addplot[mark=square*,opacity=0.8,mark options={black,fill=green},mark size=3pt]
	coordinates {
(3,0.555)
(6,0.612)
(9,0.673)
(12,0.693)
(15,0.701)
(18,0.719)
(21,0.702)
(24,0.725)
(27,0.73)
(30,0.743)};

	\addplot[mark=triangle*,opacity=0.8,mark options={black,fill=orange},mark size=5pt]
	coordinates {
(3,0.372)
(6,0.41)
(9,0.445)
(12,0.495)
(15,0.479)
(18,0.503)
(21,0.528)
(24,0.545)
(27,0.561)
(30,0.572)
	};
	\addplot[mark=triangle*,opacity=0.8,mark options={black,fill=red},mark size=5pt]
	coordinates {
(3,0.545)
(6,0.644)
(9,0.703)
(12,0.743)
(15,0.749)
(18,0.756)
(21,0.761)
(24,0.77)
(27,0.774)
(30,0.774)};

	\addplot[mark=*,opacity=0.8,mark options={black,fill=red},mark size=3pt]
	coordinates {
(3,0.526)
(6,0.651)
(9,0.687)
(12,0.729)
(15,0.733)
(18,0.74)
(21,0.743)
(24,0.756)
(27,0.768)
(30,0.762)};

	\addplot[mark=triangle*,opacity=0.8,mark options={black,fill=blue},mark size=5pt]
	coordinates {
(3,0.56)
(6,0.67)
(9,0.722)
(12,0.764)
(15,0.765)
(18,0.767)
(21,0.768)
(24,0.779)
(27,0.787)
(30,0.788)};

	\addplot[mark=*,opacity=0.8,mark options={black,fill=blue},mark size=3pt]
	coordinates {
(3,0.53)
(6,0.654)
(9,0.699)
(12,0.747)
(15,0.752)
(18,0.756)
(21,0.765)
(24,0.773)
(27,0.782)
(30,0.778)};	

	\nextgroupplot[
	title=Github,
	ytick={0.4,0.5,0.6,0.7,0.8},
	xtick={6,12,18,24,30},
	ymin=0.3,
	ymax=0.81,
	xmin=1,
	xmax=32,
	xlabel=Training samples per class]
	\addplot[mark=triangle*,opacity=0.8,mark options={black, fill=yellow},mark size=5pt]
	coordinates {
(3,0.538)
(6,0.512)
(9,0.527)
(12,0.559)
(15,0.543)
(18,0.517)
(21,0.519)
(24,0.541)
(27,0.557)
(30,0.564)};

	\addplot[mark=diamond*,opacity=0.8,mark options={black,fill=brown},mark size=5pt]
	coordinates {
(3,0.472)
(6,0.521)
(9,0.533)
(12,0.541)
(15,0.57)
(18,0.595)
(21,0.609)
(24,0.616)
(27,0.623)
(30,0.631)};

	\addplot[mark=*,opacity=0.8,mark options={black,fill=pink},mark size=3pt]
	coordinates {
(3,0.547)
(6,0.56)
(9,0.584)
(12,0.591)
(15,0.604)
(18,0.622)
(21,0.635)
(24,0.64)
(27,0.658)
(30,0.664)};

	\addplot[mark=square*,opacity=0.8,mark options={black,fill=green},mark size=3pt]
	coordinates {
(3,0.534)
(6,0.541)
(9,0.544)
(12,0.531)
(15,0.54)
(18,0.53)
(21,0.527)
(24,0.529)
(27,0.551)
(30,0.547)};

	\addplot[mark=triangle*,opacity=0.8,mark options={black,fill=orange},mark size=5pt]
	coordinates {
(3,0.574)
(6,0.644)
(9,0.661)
(12,0.68)
(15,0.706)
(18,0.733)
(21,0.745)
(24,0.751)
(27,0.739)
(30,0.748)};

	\addplot[mark=triangle*,opacity=0.8,mark options={black,fill=red},mark size=5pt]
	coordinates {
(3,0.685)
(6,0.738)
(9,0.746)
(12,0.75)
(15,0.751)
(18,0.752)
(21,0.753)
(24,0.755)
(27,0.756)
(30,0.756)};

	\addplot[mark=*,opacity=0.8,mark options={black,fill=red},mark size=3pt]
	coordinates {
(3,0.681)
(6,0.725)
(9,0.725)
(12,0.733)
(15,0.741)
(18,0.74)
(21,0.741)
(24,0.741)
(27,0.742)
(30,0.745)};

	\addplot[mark=triangle*,opacity=0.8,mark options={black,fill=blue},mark size=5pt]
	coordinates {
(3,0.637)
(6,0.744)
(9,0.727)
(12,0.75)
(15,0.77)
(18,0.76)
(21,0.754)
(24,0.761)
(27,0.76)
(30,0.746)};
	
	\addplot[mark=*,opacity=0.8,mark options={black,fill=blue},mark size=3pt]
	coordinates {
(3,0.61)
(6,0.742)
(9,0.722)
(12,0.734)
(15,0.759)
(18,0.75)
(21,0.75)
(24,0.767)
(27,0.762)
(30,0.761)};	

	\nextgroupplot[
	ytick={0.4,0.5,0.6,0.7,0.8},
	xtick={6,12,18,24,30},
	ymin=0.3,
	ymax=0.81,
	xmin=1,
	xmax=32,
	xlabel=Training samples per class,
	title = Twitch PT,
	legend style={at={(1.63,0.5)},anchor=east,legend columns=1}]
	\addplot[mark=triangle*,opacity=0.8,mark options={black,fill=yellow},mark size=5pt]
	coordinates {
(3,0.478)
(6,0.48)
(9,0.481)
(12,0.477)
(15,0.48)
(18,0.482)
(21,0.487)
(24,0.495)
(27,0.494)
(30,0.494)};\addlegendentry{TADW}%

	\addplot[mark=diamond*,opacity=0.8,mark options={black,fill=brown},mark size=5pt]
	coordinates {
(3,0.498)
(6,0.517)
(9,0.523)
(12,0.528)
(15,0.532)
(18,0.531)
(21,0.535)
(24,0.535)
(27,0.538)
(30,0.538)};\addlegendentry{AANE}%

	\addplot[mark=*,opacity=0.8,mark options={black,fill=pink},mark size=3pt]
	coordinates {
(3,0.545)
(6,0.547)
(9,0.547)
(12,0.553)
(15,0.567)
(18,0.565)
(21,0.57)
(24,0.57)
(27,0.571)
(30,0.573)};\addlegendentry{ASNE}%
	
	\addplot[mark=square*,opacity=0.8,mark options={black,fill=green},mark size=3pt]
	coordinates {
(3,0.501)
(6,0.527)
(9,0.529)
(12,0.533)
(15,0.535)
(18,0.535)
(21,0.537)
(24,0.542)
(27,0.543)
(30,0.537)};\addlegendentry{BANE}%

	\addplot[mark=triangle*,opacity=0.8,mark options={black,fill=orange},mark size=5pt]
	coordinates {
(3,0.524)
(6,0.541)
(9,0.543)
(12,0.546)
(15,0.551)
(18,0.549)
(21,0.55)
(24,0.553)
(27,0.557)
(30,0.559)};\addlegendentry{TENE}%
	\addplot[mark=triangle*,opacity=0.8,mark options={black,fill=red},mark size=5pt]
	coordinates {
(3,0.508)
(6,0.526)
(9,0.527)
(12,0.544)
(15,0.547)
(18,0.549)
(21,0.55)
(24,0.553)
(27,0.554)
(30,0.559)};\addlegendentry{AE}%
	\addplot[mark=*,opacity=0.8,mark options={black,fill=red},mark size=3pt]
	coordinates {
(3,0.497)
(6,0.528)
(9,0.539)
(12,0.552)
(15,0.554)
(18,0.554)
(21,0.555)
(24,0.555)
(27,0.556)
(30,0.559)
	};\addlegendentry{MUSAE}%
	\addplot[mark=triangle*,opacity=0.8,mark options={black,fill=blue},mark size=5pt]
	coordinates {
(3,0.507)
(6,0.532)
(9,0.537)
(12,0.547)
(15,0.553)
(18,0.553)
(21,0.557)
(24,0.557)
(27,0.562)
(30,0.564)

	};\addlegendentry{AE EGO}%
	\addplot[mark=*,opacity=0.8,mark options={black,fill=blue},mark size=3pt]
	coordinates {
(3,0.507)
(6,0.53)
(9,0.539)
(12,0.55)
(15,0.554)
(18,0.554)
(21,0.555)
(24,0.557)
(27,0.559)
(30,0.563)

	};\addlegendentry{MUSAE EGO}%
	\end{groupplot}
	\end{tikzpicture}
	\caption{The $k$-shot node classification  performance for varying $k$, evaluated by average micro $F_1$ scores in the test set over a 100 seeded train-test splits of the vertices.}\label{fig:facebook_ratio}

\end{figure*}
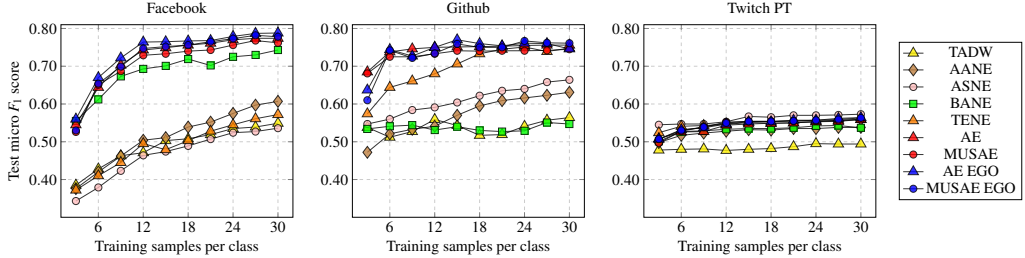

\vspace{-5mm}

We also test the \textit{few-shot learning} ability of attributed embeddings by repeating the above experiment, but training the logistic regression model with only $k$ randomly selected samples per class. Results for a representative selection of datasets and $k \!\in\! \{3, 
..., 30\}$ are shown in Figure \ref{fig:facebook_ratio}.
Our attributed embeddings show a material performance improvement at few shot learning over other unsupervised methods, particularly for the larger data sets (Facebook and Github). We also observe a modest performance benefit for the \textit{EGO} models, suggesting that the additional network structure they encode is useful when data is limited.

\begin{table}[h!]
\vspace{-5mm}
\caption{Node attribute regression with embedding features: average test $R^2$ and standard error calculated from a 100 splits for predicting monthly website traffic (best results in bold).}\label{tab:basic_evaluation}
\centering{\footnotesize
\setlength{\tabcolsep}{4pt}
 \begin{tabular}{lccc}
             &\specialcell{\textbf{Wikipedia}\\ \textbf{Chameleons}} &\specialcell{\textbf{Wikipedia}\\ \textbf{Crocodiles}} &\specialcell{\textbf{Wikipedia}\\ \textbf{Squirrels}}\\ \hline
            \text{DeepWalk}  & $.375\pm.004$ & $.553\pm.002$  &  $.170\pm.001$ \\[0.25em]
            \text{LINE}$_2$  & $.381\pm.003$ & $.586\pm.001$  &  $.232\pm.002$ \\[0.25em]
            \text{Node2Vec}  & $.414\pm.003$ & $.574\pm.001$  &  $.174\pm.002$ \\[0.25em]
            \text{Walklets}  & $.426\pm.003$ & $.625\pm.001$  &  $.249\pm.002$ \\[0.25em]
            \text{NetMF}  & $.440\pm.003$ & $.629\pm.002$  &  $.099\pm.002$ \\[0.25em]
            \text{HOPE}  & $.380\pm.002$ & $.571\pm.001$  &  $.175\pm.001$ \\[0.25em]
            \text{GraRep}  &$.520\pm.004$ & $.696\pm.002$  &  $.301\pm.001$ \\[0.25em]
            \hline 
            \text{TADW}      & $.527\pm.003$ & $.636\pm.001$   & $.271\pm.002$ \\[0.25em]
            \text{AANE}      & $.598\pm.007$ & $.732\pm.002$   & $.287\pm.002$ \\[0.25em]
            \text{ASNE}      & $.440\pm.009$ & $.572\pm.003$   & $.229\pm.005$ \\[0.25em]
            \text{BANE}      & $.464\pm.003$ & $.617\pm.001$   & $.168\pm.002$ \\[0.25em]
            \text{TENE}      & $.494\pm.020$ & $.701\pm.003$   & $.321\pm.007$ \\[0.25em]
            \hline
            \textit{AE}      &  $.642\pm.006$ &  $\usecond{.743}\pm.003$  & $.291\pm.006$ \\[0.25em]
            \textit{AE-EGO}  &$.644\pm.009$   & $.732\pm.002$   & $.283\pm.006$ \\[0.25em]
            \textit{MUSAE}   & $\mathbf{.658\pm.004}$ &   $.736\pm.003$   &$\usecond{.338}\pm.007$  \\[0.25em]
            \textit{MUSAE-EGO}&$\usecond{.653}\pm.011$ &$\mathbf{.747\pm.003}$ &$\mathbf{.354\pm.009}$\\[0.25em]
            \hline
        \end{tabular}
    }
\vspace{-5mm}
\end{table}

\subsection{Node attribute regression}

We compare the ability of our attributed embeddings against those of other unsupervised methods at predicting real valued node attributes. For each model, we train embeddings on an unsupervised basis and learn regression parameters of an elastic net to predict the log average web traffic (attribute) of each page (node) of the Wikipedia datasets (created for this task). 
Table \ref{tab:basic_evaluation} reports average test $R^2$ (explained variance) and standard error over 100 seeded (80\% train - 20\% test) splits. 
This shows that: 
(i) our attributed embeddings tend to outperform all other methods at the regression task; 
(ii) multi-scale methods (e.g. \textit{MUSAE}) tend to outperform pooled methods (e.g. \textit{AE}); 
(iii) the additional network information of the \textit{EGO} models appears beneficial, but only in the multi-scale case.
Overall, \textit{MUSAE-EGO} outperforms the best baseline for each dataset by 2-10\%.

\vspace{-5mm}
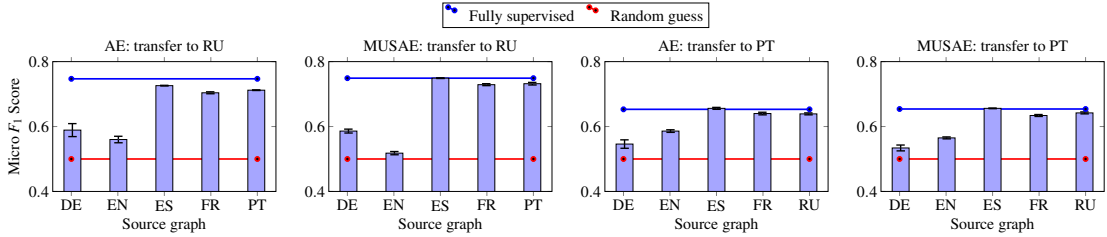
\begin{figure}[h!]
	\centering
	\begin{tikzpicture}[scale=0.4,transform shape]
	\tikzset{font={\fontsize{15pt}{12}\selectfont}}
	\begin{groupplot}[group style={group size=4 by 1,
		horizontal sep=50pt, vertical sep=70pt,ylabels at=edge left},
	grid style={dashed, gray!40},
	scaled ticks=false,
	inner axis line style={-stealth}]

  \nextgroupplot[
  title=AE: transfer to RU,
  ybar=-.58cm,
 bar width=16pt, 
  ytick ={0.0,0.2,0.4,0.6,0.8},
  height=5.9cm,
  width=9cm,
  ylabel={Micro $F_1$ Score},
  xlabel={Source graph},
  ymin=0.4, ymax=0.8,    
  xticklabels={DE, EN, ES, FR, PT},
  xtick={1,2,3,4,5},
  	legend style = { column sep = 10pt, legend columns = -1, legend to name = grouplegend} 
  ]
 \addplot[draw=blue, mark = *,ultra thick, smooth] 
    coordinates {(1.0,0.747) (5.0,0.747)};\addlegendentry{Fully supervised}
\addplot[draw=red,mark=*, ultra thick,smooth] 
    coordinates {(1.00,0.5) (5.00,0.5)};\addlegendentry{Random guess}
  \addplot [fill=blue!35,error bars/.cd,y dir=both,y explicit, error bar style={line width=1.3pt},
    error mark options={
      rotate=90,
      mark size=4pt,
      line width=1.3pt
    }] coordinates {(1,0.589)+- (0.020,0.020)};
  \addplot [fill=blue!35,error bars/.cd,y dir=both,y explicit, error bar style={line width=1.3pt},
    error mark options={
      rotate=90,
      mark size=4pt,
      line width=1.3pt
    }] coordinates {(2,0.56)+- (0.010,0.010)};
  \addplot [fill=blue!35,error bars/.cd,y dir=both,y explicit, error bar style={line width=1.3pt},
    error mark options={
      rotate=90,
      mark size=4pt,
      line width=1.3pt
    }] coordinates {(3,0.726)+- (0.001,0.001)};
  \addplot [fill=blue!35,error bars/.cd,y dir=both,y explicit, error bar style={line width=1.3pt},
    error mark options={
      rotate=90,
      mark size=4pt,
      line width=1.3pt
    }] coordinates {(4,0.704)+- (0.003,0.003)};
  \addplot [fill=blue!35,error bars/.cd,y dir=both,y explicit, error bar style={line width=1.3pt},
    error mark options={
      rotate=90,
      mark size=4pt,
      line width=1.3pt
    }] coordinates {(5,0.712)+- (0.001,0.001)};

	  \nextgroupplot[
  title=MUSAE: transfer to RU,
  ybar=-.58cm,
 bar width=16pt, 
  ytick ={0.0,0.2,0.4,0.6,0.8},
  height=5.9cm,
  width=9cm,
  xlabel={Source graph},
  ymin=0.4, ymax=0.8,    
  xticklabels={DE, EN, ES, FR, PT},
  xtick={1,2,3,4,5}
  ]
 \addplot[draw=blue, mark = *,ultra thick, smooth] 
    coordinates {(1.0,0.749) (5.0,0.749)};
\addplot[draw=red,mark=*, ultra thick,smooth] 
    coordinates {(1.00,0.5) (5.00,0.5)};
  \addplot [fill=blue!35,error bars/.cd,y dir=both,y explicit, error bar style={line width=1.3pt},
    error mark options={
      rotate=90,
      mark size=4pt,
      line width=1.3pt
    }] coordinates {(1,0.586)+- (0.006,0.006)};
  \addplot [fill=blue!35,error bars/.cd,y dir=both,y explicit, error bar style={line width=1.3pt},
    error mark options={
      rotate=90,
      mark size=4pt,
      line width=1.3pt
    }] coordinates {(2,0.518)+- (0.005,0.005)};
  \addplot [fill=blue!35,error bars/.cd,y dir=both,y explicit, error bar style={line width=1.3pt},
    error mark options={
      rotate=90,
      mark size=4pt,
      line width=1.3pt
    }] coordinates {(3,0.749)+- (0.001,0.001)};
  \addplot [fill=blue!35,error bars/.cd,y dir=both,y explicit, error bar style={line width=1.3pt},
    error mark options={
      rotate=90,
      mark size=4pt,
      line width=1.3pt
    }] coordinates {(4,0.729)+- (0.003,0.003)};
  \addplot [fill=blue!35,error bars/.cd,y dir=both,y explicit, error bar style={line width=1.3pt},
    error mark options={
      rotate=90,
      mark size=4pt,
      line width=1.3pt
    }] coordinates {(5,0.732)+- (0.004,0.004)};

  \nextgroupplot[
  title=AE: transfer to PT,
  ybar=-.58cm,
 bar width=16pt, 
  ytick ={0.0,0.2,0.4,0.6,0.8},
  height=5.9cm,
  width=9cm,
  xlabel={Source graph},
  ymin=0.4, ymax=0.8,  
  xticklabels={DE, EN, ES, FR, RU},
  xtick={1,2,3,4,5}
  ]
  
 \addplot[draw=blue, mark = *,ultra thick, smooth] 
    coordinates {(1.0,0.653) (5.0,0.653)};
\addplot[draw=red,mark=*, ultra thick,smooth] 
    coordinates {(1.00,0.5) (5.00,0.5)};
  \addplot [fill=blue!35,error bars/.cd,y dir=both,y explicit, error bar style={line width=1.3pt},
    error mark options={
      rotate=90,
      mark size=4pt,
      line width=1.3pt
    }] coordinates {(1,0.546)+- (0.013,0.013)};
  \addplot [fill=blue!35,error bars/.cd,y dir=both,y explicit, error bar style={line width=1.3pt},
    error mark options={
      rotate=90,
      mark size=4pt,
      line width=1.3pt
    }] coordinates {(2,0.586)+- (0.004,0.004) };
  \addplot [fill=blue!35,error bars/.cd,y dir=both,y explicit, error bar style={line width=1.3pt},
    error mark options={
      rotate=90,
      mark size=4pt,
      line width=1.3pt
    }] coordinates {(3,0.656)+- (0.003,0.003)};
  \addplot [fill=blue!35,error bars/.cd,y dir=both,y explicit, error bar style={line width=1.3pt},
    error mark options={
      rotate=90,
      mark size=4pt,
      line width=1.3pt
    }] coordinates {(4,0.64)+- (0.004,0.004) };
  \addplot [fill=blue!35,error bars/.cd,y dir=both,y explicit, error bar style={line width=1.3pt},
    error mark options={
      rotate=90,
      mark size=4pt,
      line width=1.3pt
    }] coordinates {(5,0.639)+- (0.003,0.003) };
  
	  \nextgroupplot[
  title=MUSAE: transfer to PT,
  ybar=-.58cm,
 bar width=16pt, 
  ytick ={0.0,0.2,0.4,0.6,0.8},
  height=5.9cm,
  width=9cm,
  xlabel={Source graph},
  ymin=0.4, ymax=0.8,  
  xticklabels={DE, EN, ES, FR, RU},
  xtick={1,2,3,4,5}
  ]
 \addplot[draw=blue, mark = *,ultra thick, smooth] 
    coordinates {(1.0,0.654) (5.0,0.654)};
\addplot[draw=red,mark=*, ultra thick,smooth] 
    coordinates {(1.00,0.5) (5.00,0.5)};
  \addplot [fill=blue!35,error bars/.cd,y dir=both,y explicit, error bar style={line width=1.3pt},
    error mark options={
      rotate=90,
      mark size=4pt,
      line width=1.3pt
    }] coordinates {(1,0.534)+- (0.009,0.009)};
  \addplot [fill=blue!35,error bars/.cd,y dir=both,y explicit, error bar style={line width=1.3pt},
    error mark options={
      rotate=90,
      mark size=4pt,
      line width=1.3pt
    }] coordinates {(2,0.565)+- (0.003,0.003) };
  \addplot [fill=blue!35,error bars/.cd,y dir=both,y explicit, error bar style={line width=1.3pt},
    error mark options={
      rotate=90,
      mark size=4pt,
      line width=1.3pt
    }] coordinates {(3,0.656)+- (0.001,0.001)};
  \addplot [fill=blue!35,error bars/.cd,y dir=both,y explicit, error bar style={line width=1.3pt},
    error mark options={
      rotate=90,
      mark size=4pt,
      line width=1.3pt
    }] coordinates {(4,0.634)+- (0.003,0.003) };
  \addplot [fill=blue!35,error bars/.cd,y dir=both,y explicit, error bar style={line width=1.3pt},
    error mark options={
      rotate=90,
      mark size=4pt,
      line width=1.3pt
    }] coordinates {(5,0.642)+- (0.003,0.003) };
  
	\end{groupplot}
	\node at ($(group c2r1) + (4.5cm,3.7cm)$) {\ref{grouplegend}}; 
	\end{tikzpicture}
\caption{Transfer learning: mean micro $F_1$ score and standard error over 10 runs with the Twitch Portugal and Russia datasets as target and each other dataset as the source graph.
Reference lines show performance for random guessing (red) and fully re-training on the target graph (blue).}
\label{tab:transfer_1}
\end{figure}
\vspace{-7mm}

\subsection{Transfer learning}
Neighbourhood based methods such as \textit{DeepWalk} \cite{deepwalk} learn node embeddings, typically with no mechanism to relate the embeddings of distinct graphs. Thus a regression model learned to predict attributes from node embeddings for one graph (as in any of the previous tasks) will not generally perform well given the node embeddings of another graph. 
The attributed models \textit{MUSAE} and \textit{AE}, however, learn representations of both nodes and features, in principle enabling information to be shared between graphs with common features. 
That is, if node and feature embeddings and a regression model are learned for a source graph $\gS$ with feature set $\sF$, then node embeddings can be learned for a target graph $\gT$ (with the same features $\sF$) that \textit{fit} with the feature embeddings learned for $\gS$ and so also its regression model. 
This amounts to \textit{transfer learning}, or zero-shot attribute prediction on the target graph. 
The information shared between graphs allows nodes of $\gT$ to be embedded into the same latent space as those of $\gS$. Since other unsupervised embedding methods do not learn explicit feature embeddings, they do not enable such transfer learning. 

The Twitch datasets contain disjoint vertex sets but a common set of features $\sF$. To perform transfer learning, we: (i) learn node and attribute embeddings for a source graph; (ii) train a regression model to predict a test attribute $f \!\notin\! \sF$; (iii) re-run the embedding algorithm on a target graph but with \textit{feature embedding parameters fixed} (as learned in step (i)); and (iv) use the regression model to predict attribute $f$ for target graph nodes. Figure \ref{tab:transfer_1} shows micro $F_1$ scores and standard error (over 10 runs) for the Twitch Portugal and Russia datasets as target and each other dataset as source graphs. The results confirm that \textit{MUSAE} and \textit{AE} learn feature embeddings that transfer between graphs with common features. Specifically, we see that performance always beats random guessing (red line) and that in some cases compares closely to re-training the feature embeddings and regression on the target graph (blue line).

\subsection{Scalability}
We demonstrate the efficiency of our algorithms with synthetic data, varying the number of nodes and features per node. Figure \ref{fig:runtime} shows the mean runtime for sets of 100 experiments on Erdos-Renyi graphs with the number of nodes as indicated, $8$ edges per node and the indicated number of unique features per node uniformly selected from a set of $2^{11}$. We used the hyperparameters in Table \ref{tab:our_params} except that we perform a single epoch with asynchronous gradient descent.
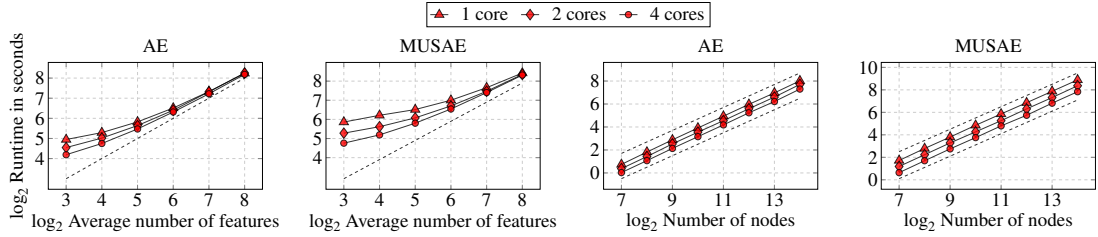
\begin{figure}[h!]
	\centering
	\begin{tikzpicture}[scale=0.4,transform shape]
	\tikzset{font={\fontsize{18pt}{12}\selectfont}}
	\begin{groupplot}[group style={group size=4 by 1,
		horizontal sep=60pt, vertical sep=60pt,ylabels at=edge left},
	width=0.6\textwidth,
	height=0.4\textwidth,
	grid=major,
	title=AE,	
	grid style={dashed, gray!40},
	scaled ticks=false,
	inner axis line style={-stealth}]
	\nextgroupplot[ytick={4,5,6,7,8},
	xtick={3,4,5,6,7,8},
	xlabel=$\log_2$ Average number of features,
	ylabel=$\log_2$ Runtime in seconds,
	enlargelimits=0.1,
	legend style = { column sep = 10pt, legend columns = -1, legend to name = grouplegend,}]
	
	\addplot[mark=triangle*,opacity=0.8,mark options={black,fill=red},mark size=5pt]
	coordinates {

(3,4.947)
(4,5.291)
(5,5.823)
(6,6.513)
(7,7.328)
(8,8.241)

	};\addlegendentry{1 core}%
	\addplot[mark=diamond*,opacity=0.8,mark options={black,fill=red},mark size=5pt]
	coordinates {
(3,4.547)
(4,5.019)
(5,5.64)
(6,6.403)
(7,7.319)
(8,8.252)

	};\addlegendentry{2 cores}%
	\addplot[mark=*,opacity=0.8,mark options={black,fill=red},mark size=3pt]
	coordinates {
(3,4.191)
(4,4.75)
(5,5.471)
(6,6.314)
(7,7.219)
(8,8.183)

	};\addlegendentry{4 cores}%
	\addplot[opacity=0.8,dashed,thick,mark options={black,fill=gray},mark size=5pt]
	coordinates {
		(3,3)
		(4,4)
		(5,5)		
		(6,6)
		(7,7)
		(8,8)

	};

	\nextgroupplot[ytick={4,5,6,7,8},
	xtick={3,4,5,6,7,8},
	xlabel=$\log_2$ Average number of features,
	enlargelimits=0.1,
	title=MUSAE,		
	legend style = { column sep = 10pt, legend columns = -1, legend to name = grouplegend,}]
	\addplot[mark=triangle*,opacity=0.8,mark options={black,fill=red},mark size=5pt]
	coordinates {

(3,5.861)
(4,6.213)
(5,6.514)
(6,7.002)
(7,7.657)
(8,8.424)

	};\addlegendentry{1 core}%

	\addplot[mark=diamond*,opacity=0.8,mark options={black,fill=red},mark size=5pt]
	coordinates {
(3,5.28)
(4,5.619)
(5,6.079)
(6,6.703)
(7,7.47)
(8,8.348)

	};\addlegendentry{2 cores}%
	\addplot[mark=*,opacity=0.8,mark options={black,fill=red},mark size=3pt]
	coordinates {
(3,4.758)
(4,5.19)
(5,5.801)
(6,6.543)
(7,7.401)
(8,8.317)

	};\addlegendentry{4 cores}%

	\addplot[opacity=0.8,dashed,thick,mark options={black,fill=gray},mark size=5pt]
	coordinates {
		(3,2.9)
		(4,3.9)
		(5,4.9)		
		(6,5.9)
		(7,6.9)
		(8,7.9)
	};

	\nextgroupplot[ytick={0,2,4,6,8,10},
	xtick={7,9,11,13,15},
	xlabel=$\log_2$ Number of nodes,
	enlargelimits=0.1,
	legend style = { column sep = 10pt, legend columns = -1, legend to name = grouplegend,}]
	
	\addplot[mark=triangle*,opacity=0.8,mark options={black,fill=red},mark size=5pt]
	coordinates {
(7,0.743)
(8,1.818)
(9,2.882)
(10,3.914)
(11,4.941)
(12,5.972)
(13,6.969)
(14,8.019)

	};\addlegendentry{1 core}%
	\addplot[mark=diamond*,opacity=0.8,mark options={black,fill=red},mark size=5pt]
	coordinates {
(7,0.376)
(8,1.426)
(9,2.53)
(10,3.548)
(11,4.55)
(12,5.588)
(13,6.604)
(14,7.705)

	};\addlegendentry{2 cores}%
	\addplot[mark=*,opacity=0.8,mark options={black,fill=red},mark size=3pt]
	coordinates {
(7,0.031)
(8,1.071)
(9,2.134)
(10,3.159)
(11,4.182)
(12,5.236)
(13,6.216)
(14,7.3)

	};\addlegendentry{4 cores}%

	\addplot[opacity=0.8,dashed,thick,mark options={black,fill=gray},mark size=5pt]
	coordinates {
		(7,-0.5)
		(8,0.5)
		(9,1.5)		
		(10,2.5)
		(11,3.5)
		(12,4.5)
		(13,5.5)
		(14,6.5)	
	};
	
	\addplot[opacity=0.8,dashed,thick,mark options={black,fill=gray},mark size=5pt]
	coordinates {
		(7,1.7)
		(8,2.7)
		(9,3.7)		
		(10,4.7)
		(11,5.7)
		(12,6.7)
		(13,7.7)
		(14,8.7)	
	};

	\nextgroupplot[ytick={0,2,4,6,8,10},
	xtick={7,9,11,13,15},
	xlabel=$\log_2$ Number of nodes,
	enlargelimits=0.1,
	title=MUSAE,		
	legend style = { column sep = 10pt, legend columns = -1, legend to name = grouplegend,}]
	\addplot[mark=triangle*,opacity=0.8,mark options={black,fill=red},mark size=5pt]
	coordinates {
(7,1.721)
(8,2.785)
(9,3.802)
(10,4.849)
(11,5.858)
(12,6.856)
(13,7.871)
(14,8.887)

	};\addlegendentry{1 core}%

	\addplot[mark=diamond*,opacity=0.8,mark options={black,fill=red},mark size=5pt]
	coordinates {
(7,1.186)
(8,2.233)
(9,3.28)
(10,4.267)
(11,5.258)
(12,6.307)
(13,7.314)
(14,8.366)

	};\addlegendentry{2 cores}%
	\addplot[mark=*,opacity=0.8,mark options={black,fill=red},mark size=3pt]
	coordinates {
(7,0.638)
(8,1.697)
(9,2.749)
(10,3.754)
(11,4.796)
(12,5.741)
(13,6.816)
(14,7.836)

	};\addlegendentry{4 cores}%
		
		\addplot[opacity=0.8,dashed,thick,mark options={black,fill=gray},mark size=5pt]
	coordinates {
		(7,0.1)
		(8,1.1)
		(9,2.1)		
		(10,3.1)
		(11,4.1)
		(12,5.1)
		(13,6.1)
		(14,7.1)	
	};
	
	\addplot[opacity=0.8,dashed,thick,mark options={black,fill=gray},mark size=5pt]
	coordinates {
		(7,2.5)
		(8,3.5)
		(9,4.5)		
		(10,5.5)
		(11,6.5)
		(12,7.5)
		(13,8.5)
		(14,9.5)	
	};

	\end{groupplot}
	\node at ($(group c2r1) + (4.5cm,3.7cm)$) {\ref{grouplegend}}; 
	\end{tikzpicture}
    \caption{Training time as a function of average feature count and number of nodes. Dashed lines are linear runtime references.}\label{fig:runtime}

\end{figure}
For both \textit{AE} and \textit{MUSAE}, we see  that (i) runtime is linear in the number of nodes and in the average number of features per node; and (ii) increasing the number of cores does not decrease runtime as the number of unique features per vertex approaches the size of the feature set. Observation (i) agrees with our complexity analysis in Section \ref{sec:complexity}.

\section{Discussion and conclusion}\label{sec:conclusion}
We present attributed node embedding algorithms that learn local feature information on a pooled (\textit{AE}) and multi-scale (\textit{MUSAE}) basis. 
We augment these models to also explicitly learn local network information (\textit{AE-EGO, MUAE-EGO}), blending the benefits of attributed and proximity-preserving algorithms. 
Results on a range of datasets show that distinguishing neighbourhood information at different scales (\textit{MUSAE} models) is typically beneficial for the downstream tasks of node attribute and link prediction; and that the supplementary network information encoded by \textit{EGO} models typically improves performance further, particularly for link prediction.
Combining both, \textit{MUSAE-EGO} outperforms other unsupervised attributed algorithms at predicting attributes and matches performance of proximity-preserving embeddings in link prediction. 
Furthermore, by learning distinct node and feature embeddings, the \textit{AE} and \textit{MUSAE} algorithms enable transfer learning between graphs with common features, as we demonstrate on real datasets.

We derive explicit pointwise mutual information matrices that each of our algorithms implicitly factorise to enable future interpretability and potential analysis of the information encoded  (e.g. as possible for \textit{Word2Vec} \cite{allen2019what}) and its use in downstream tasks. We see also that two widely used proximity-preserving algorithms \cite{deepwalk,perozzidontwalk} are special cases of our models.
All of our algorithms are efficient and scale linearly with the numbers of nodes and features per node.

\appendix

\section{Proofs}\label{app:proofs}

\lemmaplims*
\begin{proof}
    The proof is analogous to that given for Theorem 2.1 in \cite{qiu2018network}. We show that the computed statistics correspond to sequences of random variables with finite expectation, bounded variance and covariances that tend to zero as the separation between variables within the sequence tends to infinity. The Weak Law of Large Numbers (S.N.Bernstein) then guarantees that the sample mean converges to the expectation of the random variable. We first consider the special case $n \!=\! 1$, i.e. we have a single sequence $v_1, ..., v_l$ generated by a random walk (see Algorithm \ref{AE_algo}). For a particular node-feature pair ($v,f$), we let $Y_i$,  $i\in\{1, ...,l-t\}$, be the indicator function for the event $v_i=v$ and $f \!\in\! \sF_{i+r}$. Thus, we have:
    \begin{equation}
        \tfrac{\# (v,f)_\mathsmaller{{\underset{r}{\rightarrow}}}}{|\sD\mathsmaller{_{\underset{r}{\rightarrow}}}|} \ =\ \tfrac{1}{l-t}\sum_{i=1}^{l-t}Y_i        ,
    \end{equation}
    the sample average of the $Y_i$s. We also have:
    \begin{align*}
        \mathbb{E}[Y_i] 
         =\ & \tfrac{deg(v)}{c}(\mP^r\mF)_{v,f} =  \tfrac{1}{c}(\mD\mP^r\mF)_{v,f}
    \end{align*}    
    \begin{align*}    
        \mathbb{E}[Y_iY_j] 
        =\ & \textup{Prob}[v_i \!=\! v, f \!\in\! \sF_{i+r}, v_j \!=\! v, f \!\in\! \sF_{j+r}]
        \\
         = \underbrace{\tfrac{deg(v)}{c}}_{p(v_i=v\!)}&
        \underbrace{\underbrace{\mP_{:v}^r}_{p(v_{\!i\!+\!r}|v_i=v\!)} \underbrace{diag(\mF_{:f})}_{p(f\in \sF_{\!i\!+\!r}|v_{i\!+\!r}\!)}
        \underbrace{\mP^{j-(i+r)}_{:v}}_{p(v_{\!j}=v|v_{i\!+\!r}\!)}}_{p(v_{\!j}=v,f\in \sF_{i\!+\!r}|v_i=v)}
        \underbrace{\mP^r_{v:}
        \mF_{:f}}_{p(f\in \sF_{\!j\!+\!r}|v_{j}=v\!)}
    \end{align*}
for $j > i+r$.
This allows us to compute the covariance:
\begin{align}
    \text{Cov}(Y_i, Y_j) & = \mathbb{E}[Y_iY_j] - \mathbb{E}[Y_i]\,\mathbb{E}[Y_j]
    \nonumber\\
     = \tfrac{deg(v)}{c}&\mP^r_{v:}diag(\mF_{:f})
    \underbrace{\big(\mP^{j-(i+r)}_{:v} 
        - \tfrac{deg(v)}{c}\1 \big)}_{\text{tends to 0 as }j-i\rightarrow\infty}
        \mP^r_{v:}\mF_{:f} ,
\end{align}
where $\1$ is a vector of ones. The difference term (indicated) tends to zero as $j-i\to\infty$ since then $p(v_{j}=v|v_{i+r})$ tends to the stationary distribution $p(v)=\tfrac{deg(v)}{c}$, regardless of $v_{i+r}$. 
Thus, applying the Weak Law of Large Numbers, the sample average converges in probability to the expected value, i.e.:
\begin{align*}
    \tfrac{\# (v,f)_\mathsmaller{{\underset{r}{\rightarrow}}}}{|\sD\mathsmaller{_{\underset{r}{\rightarrow}}}|}
    = \tfrac{1}{l-t}\sum_{i=1}^{l-t}Y_i \overset{p}{\rightarrow} \tfrac{1}{l-t}\sum_{i=1}^{l-t}\mathbb{E}[Y_i] 
    = \tfrac{1}{c}(\mD\mP^r\mF)_{v,f}
\end{align*}
A similar argument applies to $\tfrac{\# (v,f)_\mathsmaller{{\underset{r}{\leftarrow}}}}{|\sD\mathsmaller{_{\underset{r}{\leftarrow}}}|}$, with expectation term
    $\tfrac{1}{c}(\mF^{\top}\mD\mP^r)_{f,v}$.
%
In both cases, the argument readily extends to the general setting where $n>1$ with suitably defined indicator functions for each of the $n$ random walks (see \cite{qiu2018network}).
\end{proof}
\lemmaplim*

\begin{proof}
\begin{align*}
    \tfrac{\# (v,f)_r}{|\sD_r|}
    &= 
    \tfrac{\# (v,f)_\mathsmaller{{\underset{r}{\rightarrow}}}}{|\sD_r|}
    +  \tfrac{\# (v,f)_\mathsmaller{{\underset{r}{\leftarrow}}}}{|\sD_r|} 
    \\
    & = \tfrac{1}{2}\big(\tfrac{\# (v,f)_\mathsmaller{{\underset{r}{\rightarrow}}}}{|\sD\mathsmaller{_{\underset{r}{\rightarrow}}}|}
    +  \tfrac{\# (v,f)_\mathsmaller{{\underset{r}{\leftarrow}}}}{|\sD\mathsmaller{_{\underset{r}{\leftarrow}}}|} \big) 
    \nonumber\\
    & \overset{p}{\rightarrow}
    \tfrac{1}{2}\big( \tfrac{1}{c}(\mD\mP^r\mF)_{v,f}
    + \tfrac{1}{c}(\mF^{\top}\mD\mP^r)_{f,v} \big)
    \nonumber\\
    & =
    \tfrac{1}{2c}\big( \mD\mP^r\mF
    + \mP^{r\top}\mD  \mF \big)_{v,f}
    \nonumber\\
    & =
    \tfrac{1}{2c}\big( 
    (\mD\mP^r + (\mA^{\top}\mD^{-1})^r\mD)  
    \mF \big)_{v,f}
    \nonumber\\
    & =
    \tfrac{1}{2c}\big( (\mD\mP^r
    + \mD(\mD^{-1}\mA^{\top})^r)  \mF \big)_{v,f}
    \\
%
    & =
    \tfrac{1}{c}( \mD\mP^r \mF )_{v,f}
    \ .
\end{align*}
The final step follows by symmetry of $\mA$, indicating how the Lemma can be extended to directed graphs.
\end{proof}

\section{Embedding Model Hyperparameters}\label{app:embedding}
Our purpose was a fair evaluation compared to other node embedding procedures. Because of this we used hyperparameter settings that give similar expressive power to the competing methods with respect to target matrix approximation \cite{deepwalk,node2vec,perozzidontwalk} and number of dimensions. We used the model implementations available in the open-source library \textit{Karate Club} \cite{karateclub}.
\begin{itemize}
    \item \textit{DeepWalk} \cite{deepwalk}: We used the hyperparameter settings described in Table \ref{tab:our_params}. While the original \textit{DeepWalk} model uses hierarchical softmax to speed up calculations we used a negative sampling based implementation. This way \textit{DeepWalk} can be seen as a special case of \textit{Node2Vec} \cite{node2vec} when the second-order random walks are equivalent to firs-order walks.
    \item $\textit{LINE}_2$ \cite{tang2015line}: We created 64 dimensional embeddings based on $1^{st}$ and $2^{nd}$ order proximity and concatenated these together. Other hyperparameters were taken from the original work.
    \item \textit{Node2Vec} \cite{node2vec}: Except for the \textit{in-out} and \textit{return} parameters that control the second-order random walk behavior we used the hyperparameter settings described in Table \ref{tab:our_params}. These behavior control parameters were tuned with grid search from the $\{4,2,1,0.5,0.25\}$ set using a train-validation split of $80\%-20\%$ \textit{within the training set} itself.
    \item \textit{Walklets} \cite{perozzidontwalk}: We used the hyperparameters described in Table \ref{tab:our_params} except for window size. We set a window size of 4 with individual embedding sizes of 32. This way the overall number of dimensions of the representation remained the same.
    \item \textit{NetMF} \cite{qiu2018network}: We used the hyperparameters described in the respective paper, created 128 dimensional node embeddings with a window size of 5.
    \item \textit{HOPE} \cite{hope}: We utilized the normalized neighbourhood overlap as a proximity measure and the hyperparameters described in Table \ref{tab:our_params}.
    \item \textit{GraRep} \cite{cao2015grarep}: Just as in case of the multi-scale method \textit{Walklets} \cite{perozzidontwalk} we decided to set a window size of 4 with individual embedding sizes of 32.
    \item The attributed node embedding methods \textit{AANE}, \textit{ASNE}, \textit{BANE}, \textit{TADW}, \textit{TENE} all use the hyperparameters described in the respective papers except for the dimension. We parametrized these methods such way that each embedding used in the downstream tasks is 128 dimensional.
\end{itemize}

\section{Supervised Model Hyperparameters}
\label{app:gcn}

Each model was optimized with the Adam optimizer \cite{kingma_adam_2014} with the standard moving average parameters and the model implementations are sparsity aware modifications based on PyTorch Geometric \cite{fey_lenssen}. We needed these modifications in order to accommodate the large number of vertex features -- see the unique features column in Table \ref{tab:descriptive_statistics}. Except for the \textit{GAT} model \cite{gat_iclr18} we used ReLU intermediate activation functions \cite{nair2010rectified} with a softmax unit in the final layer for classification. The hyperparameters used for the training and regularization of the neural models are listed in Table \ref{tab:gcn_param}.

\begin{table}[ht!]
\centering
\caption{Hyperparameter settings used for training the graph neural network baselines.}\label{tab:gcn_param}
{\footnotesize
\vspace{2mm}
\begin{tabular}{lc}
\hline
\textbf{Parameter}&\textbf{Value}\\
\hline
Epochs           & 200 \\
Learning rate  & 0.01   \\
Dropout&0.5\\
$l_2$ Weight regularization&0.001\\
Depth& 2\\
Filters per layer&32\\
\hline
\end{tabular}
}
\end{table}

Except for the \textit{APPNP} model each baseline uses information up to 2-hop neighbourhoods. The model specific settings (when we needed to deviate from the basic settings) are listed in Table \ref{tab:gcn_param} were the followings:
\begin{itemize}
    \item \textit{Classical GCN} \cite{kipf2017semi}: We used the standard parameter settings described in this section.
    \item \textit{GraphSAGE} \cite{graphsage_nips17}: We utilized a graph convolutional aggregator on the sampled neighbourhoods, samples of 40 nodes per source, and the standard settings.
    \item \textit{GAT} \cite{gat_iclr18}: The negative slope parameter of the leaky ReLU function was 0.2, we applied a single attention head, and used the standard hyperparameter settings.
    \item \textit{MixHop} \cite{mixhop_icml19}: We took advantage of the $0^{th}$, $1^{st}$ and $2^{nd}$ powers of the normalized adjacency matrix with 32 dimensional convolutional filters for creating the first hidden representations. This was fed to a feed-forward layer to classify the nodes.
    \item \textit{ClusterGCN} \cite{clustergcn_kdd19}: Just as \cite{clustergcn_kdd19} did, we used the \textit{METIS} procedure  \cite{karypis1998fast} to decompose the graph. We clustered the graphs into disjoint clusters, and the number of clusters was the same as the number of node classes (e.g. in case of the Facebook page-page network we created 4 clusters). For training we used the earlier described setup.
    \item \textit{APPNP} \cite{klicpera_predict_2019, bojchevski2020pprgo}: The top level feed-forward layer had 32 hidden neurons, the teleport probability was set as 0.2 and we used 20 steps for approximate personalized pagerank calculation.
    \item \textit{SGCONV} \cite{sgc_icml19}: We used the $2^{nd}$ power of the normalized adjacency matrix for training the node classifier. 
\end{itemize}

~\\ 
\bibliographystyle{comnet}
\bibliography{./musae.bib}

\end{document}